\documentclass{article}

\usepackage[utf8]{inputenc} 
\usepackage[T1]{fontenc}    
\usepackage{hyperref}       
\usepackage{url}            
\usepackage{booktabs}       
\usepackage{amsfonts}       
\usepackage{nicefrac}       
\usepackage{microtype}      
\usepackage{times}
\usepackage{graphicx}
\usepackage{subfigure}
\usepackage{algorithmic}
\usepackage{amssymb}
\usepackage{amsmath}\allowdisplaybreaks
\usepackage{amsthm}
\usepackage{mathtools}
\usepackage{bm}
\usepackage[american]{babel}
\usepackage{enumitem}
\usepackage{float,dblfloatfix}
\mathtoolsset{showonlyrefs,showmanualtags}

\usepackage{microtype}
\usepackage{graphicx}
\usepackage{subfigure}
\usepackage{booktabs} 

\usepackage{hyperref}

\usepackage[preprint]{icml2019}

\icmltitlerunning{Maximizing Monotone DR-submodular Continuous Functions by Derivative-free Optimization}

\newenvironment{myproof}[1]
{\par\noindent\textbf{Proof of #1.}\ \enspace\ignorespaces\begin{allowdisplaybreaks}}
	{\end{allowdisplaybreaks}\hspace{\stretch{1}}$\square$}

\newcommand{\norm}[1]{\left\lVert#1\right\rVert}

\newtheorem{definition}{Definition}
\newtheorem{theorem}{Theorem}
\newtheorem{lemma}{Lemma}

\newtheorem{assumption}{Assumption}
\newtheorem{remark}{Remark}

\begin{document}
	
	\twocolumn[
	\icmltitle{Maximizing Monotone DR-submodular Continuous Functions \\
		by Derivative-free Optimization}
	\icmlsetsymbol{equal}{*}
	
	\begin{icmlauthorlist}
		\icmlauthor{Yibo Zhang}{ustc}
		\icmlauthor{Chao Qian}{ustc}
		\icmlauthor{Ke Tang}{sust}
	\end{icmlauthorlist}
	
	\icmlaffiliation{ustc}{Anhui Province Key Lab of Big Data Analysis and Application, USTC, China}
	\icmlaffiliation{sust}{Shenzhen Key Lab of Computational Intelligence, SUSTech, China}
	
	\icmlcorrespondingauthor{Yibo Zhang}{zyb233@mail.ustc.edu.cn}
	\icmlcorrespondingauthor{Chao Qian}{chaoqian@ustc.edu.cn}
	\icmlcorrespondingauthor{Ke Tang}{tangk3@sustc.edu.cn}
	
	\icmlkeywords{Submodular optimization, non-convex optimization}
	
	\vskip 0.3in
	]
	
	
	
	\printAffiliationsAndNotice{}  

	\begin{abstract}
		In this paper, we study the problem of monotone (weakly) DR-submodular continuous maximization. While previous methods require the gradient information of the objective function, we propose a derivative-free algorithm LDGM for the first time. We define $\beta$ and $\alpha$ to characterize how close a function is to continuous DR-submodulr and submodular, respectively. Under a convex polytope constraint, we prove that LDGM can achieve an $(1-e^{-\beta}-\epsilon)$-approximation guarantee after $O(1/\epsilon)$ iterations, which is the same as the best gradient-based algorithm. Moreover, in some special cases, a variant of LDGM can achieve an $((\alpha/2)(1-e^{-\alpha})-\epsilon)$-approximation guarantee for (weakly) submodular functions. We also give theoretical results showing that a generalization of LDGM is robust under additive noise or in stochastic settings. Empirical results on budget allocation and maximum coverage verify the effectiveness of LDGM.
	\end{abstract}

	\section{Introduction}
	
	Submodularity, which implies the diminishing return property, is usually defined on set functions. Submodular set function maximization arises in many applications, such as maximum coverage~\cite{feige1998threshold} and influence maximization~\cite{kempe2003maximizing}, to name a few. It is NP-hard in general, and has received a lot of attentions~\cite{krause2014submodular}. A well-known result is that for maximizing monotone submodular set functions with a size constraint, the greedy algorithm, which iteratively selects one element with the largest marginal gain, can achieve the optimal approximation guarantee of $(1-1/e)$~\cite{nemhauser1978best,nemhauser1978analysis}.
	
	Meanwhile, many practical applications involve objective functions defined over the integer lattice instead of subsets, e.g., budget allocation~\cite{alon2012optimizing} and welfare maximization~\cite{kapralov2013online}. Submodularity is thus extended to functions over the integer lattice. In this case, submodularity does not imply the diminishing return property, which is called DR-submodularity~\cite{soma2016maximizing}. The latter is stronger, though they are equivalent for set functions. For monotone DR-submodular maximization with a size constraint, the greedy algorithm can achieve an $(1-1/e)$-approximation guarantee~\cite{soma2014optimal}; while for submodular functions, the generalized greedy algorithm, which can select multiple copies of the same element simultaneously in one iteration, achieves an $(1/2)(1-1/e)$-approximation guarantee~\cite{alon2012optimizing}.
	
	Recently, submodularity has been further extended to continuous domains. Submodular continuous functions are a class of generally non-convex and non-concave functions, which also appear in many applications~\cite{bian2017guaranteed}. For maximizing monotone DR-submodular continuous functions with a convex polytope constraint,~\citet{chekuri2015multiplicative} proposed a multiplicative weight update method which can achieve an $(1-1/e-\epsilon)$-approximation guarantee after $O(n^2/\epsilon^2)$ steps.~\citet{bian2017guaranteed} considered a down-closed convex constraint, and proposed a Frank-Wolfe (FW) variant algorithm which can achieve an $(1-1/e-\epsilon)$-approximation guarantee after $O(1/\epsilon)$ iterations. Later, stochastic monotone DR-submodular function maximization under a general convex constraint was studied.~\citet{hassani2017gradient} showed that the stochastic gradient ascent (SGA) method can achieve an $(1/2-\epsilon)$-approximation guarantee after $O(1/\epsilon^2)$ iterations, and~\citet{mokhtari2017conditional} proposed the stochastic continuous greedy (SCG) algorithm which can achieve an $(1-1/e-\epsilon)$-approximation guarantee after $O(1/\epsilon^3)$ iterations. Note that there were also some works focusing on submodular continuous minimization~\cite{bach2016submodular,bach2018efficient,staib2017robust} and non-monotone submodular continuous maximization~\cite{bian2017continuous,bian2017guaranteed,NIPS2018_8168}.
	
	All the above-mentioned algorithms require access to the gradients of objective functions or their unbiased estimates. Thus, they cannot be directly applied to non-differentiable functions, which arise in many natural applications~\cite{bian2017guaranteed}. For the generalized maximum coverage problem, each subset $C_i \; (1\leq i \leq n)$ has a confidence $x_i \in [0,1]$ and a monotone covering function $p_i: \mathbb{R}_{+} \rightarrow 2^{C_i}$, and the objective function is $|\cup^n_{i=1} p_i(x_i)|$. For the extended text summarization problem, each sentence $i$ has a confidence $x_i \in [0,1]$ and a monotone covering function $p_i: \mathbb{R}_{+} \rightarrow 2^C$ where $C$ denotes the set of concepts, and the objective function is $\sum_{j \in \cup_i p_i(x_i)} c_j$, where $c_j$ denotes the credit of concept $j$. These two objective functions are obviously non-differentiable. Moreover, in noisy or black-box environments where only polluted objective values can be obtained, numerical differentiation is known to be ill-posed \cite{engl1996regularization}, and thus it can be difficult to acquire good gradient estimates.
	
	In this paper, we propose the first derivative-free algorithm for the problem of maximizing monotone (weakly) DR-submodular continuous functions subject to a convex polytope constraint. The idea is to discretize the original continuous optimization problem into an optimization problem over the integer lattice by utilizing the frontier of the vertex set of the polytope, and then apply the greedy algorithm. This approach only requires the oracle access to the function value, and we call it Lattice Discretization Greedy Method (LDGM).  We introduce the notion of the submodularity ratio $\alpha \in [0,1]$ and the DR-submodularity ratio $\beta \in [0,1]$ to characterize how close a general continuous function $f$ is to submodular and DR-submodular, respectively. Our main theoretical results can be summarized as follows:\\[5pt]
	$\bullet$ For monotone (weakly) DR-submodular continuous maximization with a convex polytope constraint, we prove that LDGM can achieve an $(1-e^{-\beta}-\epsilon)$-approximation guarantee after $O(1/\epsilon)$ iterations.\\[5pt]
	$\bullet$ For monotone (weakly) submodular continuous maximization with a convex polytope constraint, we prove that in some special situations, LDGM using the generalized greedy algorithm can achieve an $((\alpha/2)(1-e^{-\alpha})-\epsilon)$-approximation guarantee after $O(1/\epsilon)$ iterations.\\[5pt]
	$\bullet$ By introducing the look-ahead and averaging techniques into LDGM, we prove that the resulting algorithm, generalized LDGM, is robust in noisy environments, including additive noise or stochastic settings.\\[5pt]
	We perform experiments on the applications of budget allocation and maximum coverage. Empirical results on real-world data sets show the superior performance of LDGM.
	
	The rest of the paper is organized as follows. Section~2 introduces the concerned problem and gives some preliminaries. In Section~3, we propose the LDGM algorithm and give its approximation guarantee. Section~4 presents the analysis of generalized LDGM under noise, and Section~5 gives the empirical studies. In Section~6, we conclude this paper.

	\section{Monotone DR-submodular Continuous Maximization}
	
	\textbf{Notation.} Let $\mathbb{R}$, $\mathbb{R}_{+}$ and $\mathbb{Z}_+$ denote the set of reals, non-negative reals and non-negative integers, respectively. For two vectors $\bm{x},\bm{y} \in \mathbb{R}^n$, let $\bm{x} \wedge \bm{y}$ and $\bm{x} \vee \bm{y}$ denote the coordinate-wise minimum and maximum, respectively, that is, $\bm{x} \wedge \bm{y}=(\min\{x_1,y_1\},\ldots,\min\{x_n,y_n\})$ and $\bm{x} \vee \bm{y}=(\max\{x_1,y_1\},\ldots,\max\{x_n,y_n\})$. We use $\norm{\cdot}$ to denote the Euclidean norm of a vector. The $i$-th unit vector is denoted by $\bm{\chi}_i$, that is, the $i$-th entry of $\bm{\chi}_i$ is 1 and others are 0; the all-zeros and all-ones vectors are denoted by $\bm{0}$ and $\bm{1}$, respectively. Let $[n]$ denote the set $\{1,2,\ldots,n\}$. We denote $conv(\cdot)$ as the convex hull of a set. For $\bm{x},\bm{y}\in \mathbb{R}^n$, we say $\bm{x}\leq \bm{y}$ if $x_i \leq y_i$ for every $i$; $\bm{x}<\bm{y}$ if $\bm{x}\leq \bm{y}$ and $x_i < y_i$ for some $i$.
	
	We study the continuous functions $f:\mathcal{X}=\prod^n_{i=1} \mathcal{X}_i \rightarrow \mathbb{R}$, where $\mathcal{X}_i$ is a compact subset of $\mathbb{R}_{+}$. A function $f:\mathcal{X} \rightarrow \mathbb{R}$ is monotone if for any $\bm{x} \leq \bm{y}$, $f(\bm{x}) \leq f(\bm{y})$. Without loss of generality, we assume that monotone functions are normalized, i.e., $f(\bm{0})=0$. For a function $f:\mathcal{X} \rightarrow \mathbb{R}$, submodularity, as presented in Definition~\ref{def-sub}, does not imply the diminishing return property, called DR-submodularity as presented in Definition~\ref{def-dr-sub}. DR-submodularity is stronger than submodularity, that is, a DR-submodular function is submodular, but not vice versa. In~\cite{bian2017guaranteed}, it has been proved that submodularity is equivalent to a weak version of DR-submodularity, as presented in Definition~\ref{def-weak-dr-sub}.
	
	\begin{definition}[Submodular~\cite{bach2016submodular}]\label{def-sub}
		A function $f:\mathcal{X} \rightarrow \mathbb{R}$ is submodular if for any $\bm{x},\bm{y} \in \mathcal{X}$,
		\begin{align}\label{eq:sub}
		&f(\bm{x})+f(\bm{y}) \geq f(\bm{x} \wedge \bm{y})+f(\bm{x} \vee \bm{y}).
		\end{align}
	\end{definition}
	
	\begin{definition}[DR-Submodular~\cite{bian2017guaranteed}]\label{def-dr-sub}
		A function $f:\mathcal{X} \rightarrow \mathbb{R}$ is DR-submodular if for any $\bm{x} \leq \bm{y}$, $k \in \mathbb{R}_+$ and $i \in [n]$,
		\begin{align}\label{eq:dr-sub}
		&f(\bm{x}+k\bm{\chi}_i)-f(\bm{x}) \geq f(\bm{y}+k\bm{\chi}_i)-f(\bm{y}).
		\end{align}
	\end{definition}
	
	\begin{definition}[Weak DR-Submodular~\cite{bian2017guaranteed}]\label{def-weak-dr-sub}
		A function $f:\mathcal{X} \rightarrow \mathbb{R}$ is weak DR-submodular if for any $\bm{x} \leq \bm{y}$, $k \in \mathbb{R}_+$ and $i \in [n]$ with $x_i=y_i$,
		\begin{align}\label{eq:dr-sub}
		&f(\bm{x}+k\bm{\chi}_i)-f(\bm{x}) \geq f(\bm{y}+k\bm{\chi}_i)-f(\bm{y}).
		\end{align}
	\end{definition}
	
By Definitions~\ref{def-dr-sub}-\ref{def-weak-dr-sub} and using the equivalence between submodularity and weak DR-submodularity, we define the submodularity ratio $\alpha$ as well as the DR-submodularity ratio $\beta$, which measure to what extent a general continuous function $f$ has submodular and DR-submodular properties, respectively. They are generalizations of that for functions over the integer lattice~\cite{qian2018multiset}.
	
	\begin{definition}[Submodularity Ratio]\label{sub_ratio} The submodularity ratio of a continuous function $f:\mathcal{X} \to \mathbb{R}$ is defined as
		\begin{align}
		\alpha = \inf_{\bm{x},\bm{y} \in \mathcal{X}: \bm{x}\leq \bm{y}, k\in \mathbb{R}_+, i\in [n]: x_i=y_i} \frac{f(\bm{x}+k\bm{\chi}_i)-f(\bm{x})}{f(\bm{y}+k\bm{\chi}_i)-f(\bm{y})}.
		\end{align}
	\end{definition}
	
	\begin{definition}[DR-Submodularity Ratio]\label{sub-dr-ratio} The DR-submodularity ratio of a continuous function $f:\mathcal{X} \to \mathbb{R}$ is defined as
		\begin{align}\label{dr_ratio}
		\beta = \inf_{\bm{x},\bm{y} \in \mathcal{X}: \bm{x}\leq \bm{y}, k\in \mathbb{R}_+, i\in [n]} \frac{f(\bm{x}+k\bm{\chi}_i)-f(\bm{x})}{f(\bm{y}+k\bm{\chi}_i)-f(\bm{y})}.
		\end{align}
	\end{definition}
	
	Note that in~\cite{hassani2017gradient}, the notion of DR-submodularity ratio for differentiable functions has been also defined as
	\begin{equation}
	\gamma= \inf_{\bm{x},\bm{y} \in \mathcal{X}: \bm{x}\leq \bm{y},i\in [n]} \frac{[\nabla f(\bm{x})]_i}{[\nabla f(\bm{y})]_i},
	\end{equation}
	where $[\nabla f(\bm{x})]_i=\frac{\partial f(\bm{x})}{\partial x_i}$ is the $i$-th component of the gradient. Our definition $\beta$ does not require the differentiable property. Furthermore, it can be verified that $\beta$ and $\gamma$ are equivalent when $f$ is differentiable in $\mathcal{X}$. The proof is provided in the supplementary material.
	
	It is easy to see that $\beta\leq \alpha$. Note that $(f(\bm{x}+k\bm{\chi}_i)-f(\bm{x}))/(f(\bm{y}+k\bm{\chi}_i)-f(\bm{y}))$ reaches 1 by letting $\bm{x}=\bm{y}$, and thus $\beta \leq \alpha \leq 1$. For a monotone continuous function $f$, we make the following observations:
	\begin{remark}\label{remark1}
		For a monotone continuous function $f:\mathcal{X} \rightarrow \mathbb{R}$, it holds that 1) $0\leq \beta \leq \alpha \leq 1$; 2) $f$ is submodular iff $\alpha=1$; 3) $f$ is DR-submodular iff $\beta=\alpha=1$.
	\end{remark}
	
	Our studied problem shown in Definition~\ref{def_problem} is to maximize a monotone continuous function $f$ in a convex polytope $\mathcal{P}$. A convex polytope in $\mathbb{R}^n$ has two equivalent definitions, i.e., $\mathcal{V}$-type polytope defined by the convex hull of a finite number of points, and $\mathcal{H}$-type polytope defined by the intersection of a finite number of half-spaces. There exist algorithms to convert from one to the other, e.g., the reverse
	search method~\cite{avis1996reverse} can find all vertices of a $\mathcal{H}$-type polytope in time linear to the product of the number of vertices, the number of half-spaces and the dimension.
	
	\begin{definition}[The General Problem]\label{def_problem} Given a convex polytope $\mathcal{P}=conv(E)\subseteq \mathbb{R}^n_+$, where $E$ is a set of points located in the positive space, to maximize a monotone function $f:\mathcal{X} \to \mathbb{R}$ in $\mathcal{P}$, i.e.,
		\begin{align}\label{gen_P}
		\mathop{\arg\max}\nolimits_{\bm{x} \in \mathcal{P}} f(\bm{x}). \tag{1}
		\end{align}
	\end{definition}

	\section{The Proposed Approach}
	
	In this section, we propose the derivative-free algorithm LDGM for maximizing monotone (weakly) DR-submodular continuous functions with a convex polytope constraint. The procedure of LDGM is presented in Algorithm~\ref{simpleLDGM}. It first selects those points in $Frontier(E)$ shown in Definition~\ref{def:frontier}, and divides each point by $l$ to form a set $\mathcal{E}$ (i.e., line~1). Note that $Frontier(E)$ includes all points in $E$ that lie in $Frontier(\mathcal{P})$, and can be obtained by comparing each pair of points in $E$, with a time complexity of $O(|E|^2)$. After that, the algorithm iteratively adds one point from $\mathcal{E}$ with the largest marginal gain until $l$ points are selected (i.e., lines~3-7). Note that $l$ is a parameter that controls the iteration number and step size.
	\begin{definition}[Frontier] \label{def:frontier} The frontier of a set $E\subseteq \mathbb{R}^n_+$ is defined as
		\begin{align}
		Frontier(E)=\{\bm{x}\mid \bm{x}\in E, \nexists \bm{y}\in E:\bm{y}>\bm{x}\}.
		\end{align}
	\end{definition}

	\begin{algorithm}[t] 
		\caption{LDGM Algorithm} 
		\label{simpleLDGM} 
		\textbf{Input:} A monotone function $f: \mathcal{X} \to  \mathbb{R}$, and a convex polytope $\mathcal{P}=conv(E)$\\
		\textbf{Parameter:} Number $l$ of steps\\
		\textbf{Output:} $\bm{x} \in \mathcal{P}$	
		\begin{algorithmic}[1]
			\STATE {$\mathcal{E}:= \{ \frac{1}{l}\bm{x} \mid  \bm{x} \in Frontier(E) \}$ }
			\STATE $\bm{x}_0:=\bm{0}$ and $t:=0$
			\WHILE {$t<l$}
			\STATE $\bm{e}^*:=\arg\max_{\bm{e}\in\mathcal{E}}f(\bm{x}_t+\bm{e})-f(\bm{x}_t)$
			\STATE $\bm{x}_{t+1}:=\bm{x}_t+\bm{e}^*$
			\STATE $t:=t+1$
			\ENDWHILE
			\STATE \textbf{return} $\bm{x}_l$
		\end{algorithmic}
	\end{algorithm}
	
	From the procedure of LDGM, we can see that it discretizes the original continuous optimization problem into an optimization problem over the integer lattice $\mathbb{Z}^{\mathcal{E}}_+$, and then applies the greedy algorithm. To the best of our knowledge, LDGM is the first algorithm which does not require the information of gradient. Previous algorithms, e.g., FW~\cite{bian2017guaranteed} and SCG~\cite{mokhtari2017conditional}, require the gradient or its unbiased estimate.
	
	There is actually some relation between LDGM and FW. The update policy of LDGM in each iteration is to choose a vertex with the maximum marginal gain, while FW is to solve a linear programming problem. Note that in polytope settings, the update policy of FW can also be seen as to find a vertex using the information of gradient, as it will find a vertex by linear programming, which is used as the update direction. Moreover, $f(\bm{x}_t + \bm{e}) - f(\bm{x}_t) =\langle \nabla f(\bm{x}_t),\bm{e}\rangle$ as $l\to \infty$, showing that LDGM and FW perform similarly with infinitesimal step size.

	\subsection{Approximation Guarantee}
	
	Before giving the approximation guarantee of LDGM, we make two assumptions. First, we assume that the objective function $f$ satisfies the Lipschitz condition in $Frontier(\mathcal{P})$, which does not require $f$ to be differentiable. Note that previous algorithms~\cite{bian2017guaranteed,hassani2017gradient,mokhtari2017conditional} require that $\nabla f$ satisfies the Lipschitz condition everywhere in $\mathcal{P}$.
	
	
	\begin{assumption}\label{assumption1}
		The function $f$ is Lipschitz continuous with constant $L$ in $Frontier(\mathcal{P})$, i.e., $\forall \bm{x}, \bm{y} \in Frontier(\mathcal{P})$, $|f(\bm{x})-f(\bm{y})|\leq L\cdot \norm{\bm{x}-\bm{y}}$.
	\end{assumption}
	
	Next, we make an assumption on the radius of $\mathcal{P}$. Note that the convex body $\mathcal{P}$ is bounded, otherwise the maximization is meaningless given the monotonicity of the function $f$.
	
	\begin{assumption} \label{assumption2} There exists some $D\in \mathbb{R}_+$ such that $\forall \bm{x}\in \mathcal{P}: \norm{\bm{x}} \leq D $.
	\end{assumption}
	
	Let $\mathrm{OPT}$ denote the optimal function value. We show in Theorem~\ref{the1_LDGM} that LDGM can achieve the approximation guarantee of $(1-e^{-\beta}-\epsilon)$ after $O(1/\epsilon)$ iterations.
	
	\begin{theorem}\label{the1_LDGM}
		For maximizing monotone (weakly) DR-submodular continuous functions with a convex polytope constraint $\mathcal{P}=conv(E)$, LDGM finds $\bm{x} \in \mathcal{P}$ with $$ f(\bm{x}) \geq (1-e^{-\beta})\cdot \mathrm{OPT} - \frac{(1-e^{-\beta})mDL}{l},$$ where $m\leq |E|$.
	\end{theorem}
	
	The detailed proof is provided in the supplementary material. Here, we give some lemmas and introduce the proof intuition. Lemma~\ref{lemma1}, used in the proof of Lemma~\ref{lemma2}, extends the definition of DR-submodularity ratio to show that the diminishing return property holds for arbitrary vector increment. Lemma~\ref{lemma2} shows that for the optimization over the integer lattice $\mathbb{Z}^{\mathcal{E}}_+$, there always exists a point from $\mathcal{E}$, whose inclusion can bring an improvement on $f$ proportional to the current distance to the optimum on the lattice.
	
	\begin{lemma}\label{lemma1} For a monotone continuous function $f:\mathcal{X} \to \mathbb{R}$, $\bm{x}, \bm{y} \in \mathcal{X}$ with $\bm{x}\leq \bm{y}$, and any vector $\bm{v}\in \mathbb{R}^n_+$, we have
		\begin{align}
		\frac{f(\bm{x}+\bm{v})-f(\bm{x})}{f(\bm{y}+\bm{v})-f(\bm{y})} \geq \beta.
		\end{align}
	\end{lemma}
	
	\begin{lemma}\label{lemma2} Let $\bm{v}^*$ be the best solution one can achieve using $l$ vectors in $\mathcal{E}$, denoted as $\bm{v}^*=\sum\nolimits_{i=1}^l \bm{e}_i$, where $\bm{e}_i \in \mathcal{E}$. For any $ \bm{x}\in \mathcal{X}$, there exists $\bm{e}^*\in \mathcal{E}$ such that
		\begin{align}
		f(\bm{x}+\bm{e}^*)-f(\bm{x})\geq \frac{\beta}{l}\left(f(\bm{v}^*)-f(\bm{x})\right).
		\end{align}
	\end{lemma}
	
	Based on Lemma~\ref{lemma2}, we can prove an approximation guarantee w.r.t. the best solution on the lattice $\mathbb{Z}^{\mathcal{E}}_+$. To further bound the difference between the best solution on the lattice and a global optimal solution $\bm{x}^*$, we give Lemma~\ref{lemma_bound}, which shows that there must exist a solution $\bm{v}'$, in $Frontier(\mathcal{P})$ and also on the lattice, close enough to $\bm{x}^*$. This suggests that the Lipschitz assumption is only required in $Frontier(\mathcal{P})$. Lemma~\ref{lemma_conv} presents a geometry result w.r.t. $Frontier(\mathcal{P})$, which will be used in the proof of Lemma~\ref{lemma_bound}.
	
	\begin{lemma}\label{lemma_conv}
		Let $X=\{ \bm{x}_1,\ldots,\bm{x}_m \}$, where $\forall i\in [m], \bm{x}_i \in \mathcal{P}$. If there exist $\theta_1,\ldots,\theta_m>0$ such that $\sum_{i=1}^{m} \theta_i=1$ and $\bm{x}' =\sum_{i=1}^{m} \theta_i \bm{x}_i\in Frontier(\mathcal{P})$, then we have $conv(X)\subseteq Frontier(\mathcal{P})$.
	\end{lemma}

	\begin{lemma} \label{lemma_bound} Let $\bm{x}^*\in Frontier(\mathcal{P})$ denote a global optimal solution and $|\mathcal{E}|=m$. There exist $\bm{e}'_1,\ldots,\bm{e}'_l \in \mathcal{E}$ such that $\bm{v}'=\sum_{i=1}^{l}\bm{e}'_i\in Frontier(\mathcal{P})$ and $\norm{\bm{x}^*-\bm{v}'}\leq mD/l$.
	\end{lemma}

	Theorem~\ref{the1_LDGM} can be proved by Lemmas~\ref{lemma2} and~\ref{lemma_bound}. Lemma~\ref{lemma2} implies that in each iteration of LDGM, $f(\bm{x}_{t+1})-f(\bm{x}_t)\geq \frac{\beta}{l}\left(f(\bm{v}^*)-f(\bm{x}_t)\right)$. By induction on $t$, an approximation guarantee w.r.t. $f(\bm{v}^*)$ can be obtained. Lemma~\ref{lemma_bound} and Assumption~\ref{assumption1} lead to the relation between $f(\bm{v}^*)$ and $f(\bm{x}^*)$: $f(\bm{v}^*) \geq f(\bm{v}') \geq f(\bm{x}^*)-mDL/l$. Combining them, Theorem~\ref{the1_LDGM} holds.
	
	
	
	Next, we consider a special case. If the set $\mathcal{E}$ generated in line~1 of Algorithm~\ref{simpleLDGM} is an orthogonal set, i.e., the inner product of any two vectors from $\mathcal{E}$ equals zero, a variant of LDGM in Algorithm~\ref{LDGM_genGreedy}, called LDGM-G, is able to handle the weakly submodular case. For the optimization over the integer lattice $\mathbb{Z}^{\mathcal{E}}_+$, LDGM-G applies the generalized greedy algorithm~\cite{alon2012optimizing}, rather than the greedy algorithm. That is, in each iteration, it selects a combination $(\bm{e},j)$ (where $\bm{e} \in \mathcal{E}$) such that the average marginal gain by adding $j$ copies of $\bm{e}$ is maximized.
	
	
	\begin{algorithm} 
		\caption{LDGM-G Algorithm} 
		\label{LDGM_genGreedy} 
		\textbf{Input:} A monotone function $f: \mathcal{X} \to  \mathbb{R}$, a convex polytope $\mathcal{P}=conv(E)$, and an orthogonal set $\mathcal{E}= \{ \frac{1}{l}\bm{x} \mid \bm{x} \in Frontier(E) \}$ \\
		\textbf{Parameter:} Number $l$ of steps\\
		\textbf{Output:} $\bm{x}\in \mathcal{P}$
		\begin{algorithmic}[1]
			\STATE $\bm{x}_0:=\bm{0}$ and $t:=0$
			\STATE $\hat{\bm{x}}:=\arg\max_{\bm{e}\in \mathcal{E}}  f(l\bm{e})$
			\WHILE {True}
			\STATE $(\bm{e}^*,k):=\arg\max_{(\bm{e},j)} \left\{\frac{ f(\bm{x}_t+j\bm{e})-f(\bm{x}_t)}{j} \right\}$ \\
			\quad\quad\quad\quad\quad $s.t.$ \; $\bm{e}\in\mathcal{E}, j\in \mathbb{Z}_+: j\leq l - \frac{\langle \bm{x}_t, \bm{e} \rangle}{\langle \bm{e}, \bm{e} \rangle} $
			\IF {$t+k>l$}
			\STATE $\bm{x}_l:=\bm{x}_t+(l-t)\bm{e}^*$
			\STATE \textbf{return} $\arg\max_{\bm{x} \in \{\bm{x}_l,\hat{\bm{x}}\}} f(\bm{x})$
			\ELSE
			\STATE $\bm{x}_{t+k}:=\bm{x}_t+k\bm{e}^*$
			\STATE $t:=t+k$
			\ENDIF
			\ENDWHILE
		\end{algorithmic}
	\end{algorithm}
	
	Theorem~\ref{the2_LDGM} shows that LDGM-G achieves the approximation guarantee of $((\alpha/2)(1-e^{-\alpha})-\epsilon)$ after $O(1/\epsilon)$ iterations. The proof is provided in the supplementary material.
	
	\begin{theorem}\label{the2_LDGM}
		For maximizing monotone (weakly) submodular continuous functions with a convex polytope constraint $\mathcal{P}=conv(E)$, if $ \mathcal{E}= \{ \frac{1}{l}\bm{x} \mid \bm{x} \in Frontier(E) \}$ is an orthogonal set, LDGM-G finds $\bm{x} \in \mathcal{P}$ with $$f(\bm{x}) \geq \frac{\alpha}{2}(1-e^{-\alpha})\cdot \mathrm{OPT}-\frac{\alpha (1-e^{-\alpha})mDL}{2l},$$ where $m \leq |E|$.
	\end{theorem}

	\subsection{Discussion}
	
	LDGM works well when the constraint is given by a $\mathcal{V}$-type polytope. Although in general the number of vertices, i.e., $|E|$, can be exponential given a $\mathcal{H}$-type polytope, several common classes of $\mathcal{H}$-type polytopes have limited number of vertices. One example is $\mathcal{P}=\{\bm{x} \mid |\bm{x}|_1 \leq b, \bm{x} \geq \bm{0}\}$ with $(n+1)$ vertices. In fact for any $\bm{a} > \bm{0}$, $\mathcal{P}=\{\bm{x} \mid \bm{a}^T\bm{x} \leq b, \bm{x} \geq \bm{0}\}$ has $(n+1)$ vertices; for any $\mathbf{A}\in \mathbb{R}^{r\times n}_+$ and $r=O(1)$, $\mathcal{P}=\{\bm{x} \mid \mathbf{A} \bm{x} \leq \bm{b}, \bm{x}\geq \bm{0}\}$ has $O(n^r)$, i.e., polynomial, vertices.
	
	Moreover, although the common convex polytope constraint $\mathcal{P}=\{\bm{x}\mid \bm{a}^T\bm{x} \leq b, \bm{0}\leq \bm{x} \leq \bm{c}\}$ where $\bm{a}>\bm{0}$ can have exponential number of vertices, LDGM with a slight modification can still obtain the $(1-1/e-\epsilon)$-approximation guarantee after $O(1/\epsilon)$ iterations. The modification along with the proof are provided in the supplementary material.
	
	In addition, in the situations where the constraint is a smooth convex body, we can use a polytope with limited number of vertices to approximate the original convex body~\cite{gruber1993aspects,bronstein2008approximation}, and then apply LDGM.
	
	\section{Analysis under Noise}\label{sec:noise}
	
	Various noisy environments have been considered in the submodular optimization context~\cite{hassidim2017submodular,qian2017subset}. In the continuous submodular literatures, stochastic setting has been considered~\cite{hassani2017gradient,mokhtari2017conditional}. In this section, we consider stochastic as well as additive noise settings. To make LDGM robust against noise, we introduce two techniques, i.e., look-ahead and averaging, leading to the generalized LDGM algorithm presented in Algorithm~\ref{LDGM}. We use $\tilde{f}(\cdot)$ to represent the noisy version of $f(\cdot)$ through this section.
	
	\begin{algorithm}[t] 
		\caption{Generalized LDGM Algorithm} 
		\label{LDGM} 
		\textbf{Input:} A monotone function $f: \mathcal{X} \to  \mathbb{R}$, and a convex polytope $\mathcal{P}=conv(E)$\\
		\textbf{Parameter:} Number $l$ of steps, lookahead parameter $\gamma$ and averaging parameter $\rho_t$\\
		\textbf{Output:} $\bm{x} \in \mathcal{P}$	
		\begin{algorithmic}[1]
			\STATE {$\mathcal{E}:= \{ \frac{1}{l}\bm{x} \mid  \bm{x} \in Frontier(E) \}$ }
			\STATE $\bm{x}_0:=\bm{0}$ and $t:=0$
			\STATE $d_{-1}^{(\bm{e})}:=0$ for $\bm{e}\in \mathcal{E}$
			\WHILE {$t<l$}
			\STATE $\tilde{\Delta}_{t}^{(\bm{e})}=\tilde{f}(\bm{x}_t+\gamma\bm{e})-\tilde{f}(\bm{x}_t)$ for $\bm{e}\in \mathcal{E}$
			\STATE $d_{t}^{(\bm{e})}:=(1-\rho_t)d_{t-1}^{(\bm{e})}+\rho_t \tilde{\Delta}_{t}^{(\bm{e})}$ for $\bm{e}\in \mathcal{E}$
			\STATE $\bm{e}^*:=\arg\max_{\bm{e}\in\mathcal{E}} d^{(\bm{e})}_{t}$
			\STATE $\bm{x}_{t+1}:=\bm{x}_t+\bm{e}^*$
			\STATE $t:=t+1$
			\ENDWHILE
			\STATE \textbf{return} $\bm{x}_l$
		\end{algorithmic}
	\end{algorithm}
	
	By setting the look-ahead parameter $\gamma$ in line~5 greater than 1, the algorithm can see further than its actual step in line~8. The intuition is that the differences among all $\bm{e}\in \mathcal{E}$ can be sufficiently large to avoid being interfered by additive noise. To deal with the stochastic setting where only unbiased estimation of the true objective function is available, we use an averaging technique inspired by SCG~\cite{mokhtari2017conditional}. That is, the algorithm selects a step $\bm{e}^*$ based on the average marginal gain with historical information, i.e., $d_t^{(\bm{e})}$, rather than the current noisy marginal gain, i.e, $\tilde{\Delta}_{t}^{(\bm{e})}$. It can be seen that LDGM in Algorithm~\ref{simpleLDGM} is a special case of generalized LDGM in Algorithm~\ref{LDGM} with $\gamma=1$ and $\rho_t=1$. Detailed proofs in this section are provided in the supplementary material due to space limitations.
	
	\subsection{Additive Noise}\label{sec-additive-noise}
	
	In this subsection, we theoretically compare generalized LDGM with FW~\cite{bian2017guaranteed} under additive noise.
	
	Let $\bm{x}^*$ be an optimal solution, and $\bm{x}_t$ be the solution after $t$ iterations of FW. Let $\bm{v}^*_t=\bm{x}^*\vee \bm{x}_t - \bm{x}_t$. Assume that in the $t$-th iteration of FW, a gradient call at $\bm{x}_t$ introduces a noise term $\bm{\epsilon}_t$. That is, $\nabla f(\bm{x}_t)+\bm{\epsilon}_t$ instead of $\nabla f(\bm{x}_t)$ is used in the update policy. Denote $\bm{v}_t$ as the vector chosen in the $t$-th iteration of FW, i.e., $\bm{v}_t=\arg\max_{\bm{v}\in \mathcal{P}} \langle \bm{v},\nabla f(\bm{x}_t)+\bm{\epsilon}_t\rangle$. Suppose FW performs $l$ iterations and uses a constant step size $1/l$. Then, we can show the approximation guarantee of FW with the Lipschitz continuous assumption on the gradient of $f$, i.e., Assumption~\ref{assump:Lips_1}. The proof of Theorem~\ref{the_noise_FW} is inspired from that of Theorem~1 in~\cite{bian2017guaranteed}.

	\begin{assumption}\label{assump:Lips_1}
		The gradient of $f$ is Lipschitz continuous with constant $L_1$ in $\mathcal{P}$, i.e., $\forall \bm{x}, \bm{y} \in \mathcal{P}$, \begin{align}
		||\nabla f(\bm{x})-\nabla f(\bm{y})|| \leq L_1 ||\bm{x}-\bm{y}||.
		\end{align}
	\end{assumption}
	
	\begin{theorem}\label{the_noise_FW}
		Under additive noise, for maximizing monotone (weakly) DR-submodular continuous functions with a convex polytope constraint $\mathcal{P}=conv(E)$, FW finds $\bm{x} \in \mathcal{P}$ with\vspace{-1em}
		\begin{align}
		f(\bm{x}) &\geq (1-e^{-\beta}) \cdot \mathrm{OPT}- \frac{L_1}{2l} \\
		& \quad - \frac{1}{l} \sum_{t=0}^{l-1} (1-\beta/l)^{l-1-t} \langle \bm{v}_t-\bm{v}^*_t, \bm{\epsilon}_t   \rangle.
		\end{align}
	\end{theorem}
	
	For generalized LDGM, we set $\gamma>1$ to mitigate influences from additive noise, and set $\rho_t=1$, i.e., not using the averaging technique. Let $\mathcal{E}=\{\bm{e}_1,\ldots,\bm{e}_m \}$. In each iteration of generalized LDGM, it will make calls to obtain $f(\bm{x}_t+\bm{e}_i)$ for each $\bm{e}_i \in \mathcal{E}$. Assume that in the $t$-th iteration, a noise of $\epsilon_{t,i}$ is introduced to $f(\bm{x}_t+\bm{e}_i)$, i.e., the actual function value returned is $\tilde{f}(\bm{x}_t+\bm{e}_i)=f(\bm{x}_t+\bm{e}_i)+\epsilon_{t,i}$.
	
	Similarly to Lemma~\ref{lemma2}, Lemma~\ref{lemma_noise_LDGM} shows that even under additive noise, the improvement on $f$ in each step is proportional to the current distance to the optimum on the lattice. The difference, aside from noise, is that the best solution $\bm{v}^*$ on the lattice is substituted by the best solution on a "sparser" lattice, i.e., $\bm{v}^*_\gamma$ in Lemma~\ref{lemma_noise_LDGM}. To simplify the proof, we assume $\gamma$ is an integer and $l$ is divisible by $\gamma$.
	
	\begin{lemma}\label{lemma_noise_LDGM}
		Let $\bm{v}^*_\gamma$ be the best solution one can achieve using $l/\gamma$ vectors in $\gamma\cdot \mathcal{E}$, denoted as $\bm{v}^*_\gamma=\sum\nolimits_{j=1}^{l/\gamma} \gamma\bm{e}_{i_j}$, where $\bm{e}_{i_j} \in \mathcal{E}$. For any $ \bm{x}_t$, $\bm{e}_{i^*_t}\in \mathcal{E}$ found in line~7 of Algorithm~\ref{LDGM} satisfies
		\begin{align}
		& f(\bm{x}_t+\bm{e}_{i^*_t})-f(\bm{x}_t)\\
		& \geq \frac{\beta^2}{l}\left(f(\bm{v}^*_\gamma)-f(\bm{x}_t)\right)-\frac{\beta}{l}\sum_{j=1}^{l/\gamma} (\epsilon_{t,i^*_t}-\epsilon_{t,i_j}).
		\end{align}
	\end{lemma}
	
	By combining Lemmas~\ref{lemma_bound} and~\ref{lemma_noise_LDGM}, we can exactly follow the proof of Theorem~\ref{the1_LDGM} to derive the approximation guarantee of generalized LDGM under additive noise, i.e., Theorem~\ref{the_noise_LDGM}.
	
	\begin{theorem}\label{the_noise_LDGM}
		Under additive noise, for maximizing monotone (weakly) DR-submodular continuous functions with a convex polytope constraint $\mathcal{P}=conv(E)$, generalized LDGM with $\rho_t=1$  finds $\bm{x} \in \mathcal{P}$ with
		\begin{align}
		f(\bm{x}) &\geq (1-e^{-\beta^2})\cdot \mathrm{OPT} - \frac{mDL(1-e^{-\beta^2})\gamma}{l} \\
		&\quad -\frac{\beta}{l}\sum_{t=0}^{l-1} (1-\beta^2 / l)^{l-1-t} \sum_{j=1}^{l/\gamma} (\epsilon_{t,i^*_t}-\epsilon_{t,i_j}),
		\end{align}
where $m\leq |E|$, and $\epsilon_{t,i^*_t}, \epsilon_{t,i_j}$ are introduced in Lemma~\ref{lemma_noise_LDGM}.
	\end{theorem}
	
	We provide the comparison between FW and generalized LDGM for DR-submodular cases (i.e., $\beta=1$) in the supplementary material, showing that as the noise $\epsilon$ decreases, the error term of FW shrinks with $O(\sqrt{\epsilon})$ speed, whereas that of generalized LDGM shrinks with $O(\epsilon)$ speed. The approximation bound of generalized LDGM in Theorem~\ref{the_noise_LDGM} involves two terms w.r.t. $\gamma$, between which we can make trade-off to mitigate influences from additive noise.

	\subsection{Stochastic Setting}
	
	Recently, stochastic monotone DR-submodular continuous maximization has been studied. In the stochastic setting, the objective function $f(\bm{x})=\mathbb{E}_i[f_i(\bm{x})]$ is in an expectation form, and only unbiased estimation $\tilde{f}(\cdot)$ can be obtained. For example, when $f_i(\cdot)$ is chosen from a very large ground set, the accurate computation of its expectation is prohibitive. It has been shown that FW can perform badly in stochastic settings~\cite{hassani2017gradient}, whereas SCG can achieve an $(1-1/e-\epsilon)$-approximation guarantee in expectation within $O(1/\epsilon^3)$ iterations~\cite{mokhtari2017conditional}.
	
	In this subsection, we show in Theorem~\ref{the_stoch} that generalized LDGM can perform as well as SCG in stochastic settings, that is, generalized LDGM can achieve an $(1-1/e-\epsilon)$-approximation guarantee in expectation within $O(1/\epsilon^3)$ iterations. As in~\cite{mokhtari2017conditional}, we need Assumptions~\ref{assump:Lips_1} and~\ref{assump:stoch_var}. Assumption~\ref{assump:Lips_1} gives the Lipschitz condition on the gradients, and Assumption~\ref{assump:stoch_var} bounds the variance of the gradients, implying that the stochastic setting is not too random. Note that the two assumptions are shown in a gradient form for the sake of better understandability, though generalized LDGM does not require gradients. In fact, discrete versions of these two assumptions could simplify the proof.

	\begin{assumption}\label{assump:stoch_var}
		The variance of unbiased stochastic gradients is bounded by $\sigma^2$, i.e.,
		\begin{align}
		\mathbb{E}\left[||\nabla \tilde{f}(\bm{x})-\nabla f(\bm{x})||^2 \right] \leq \sigma^2.
		\end{align}
	\end{assumption}
	
	\begin{theorem}\label{the_stoch}
		In stochastic settings, for maximizing monotone (weakly) DR-submodular continuous functions with a convex polytope constraint $\mathcal{P}=conv(E)$, generalized LDGM with $\gamma=1 $ and $\rho_t = \frac{4}{(t+8^{2/3})}$ finds $\bm{x} \in \mathcal{P}$ with
		$$ \mathbb{E}[f(\bm{x})] \geq (1-e^{-\beta})\left(\mathrm{OPT}-\frac{mDL}{l}\right) - \frac{2Q^{\frac{1}{2}}}{l^{\frac{1}{3}}}, $$
		where $m\leq |E|$, $Q\leq \max\{L^2D^29^{2/3}, 16\hat{\sigma}^2+12L_1^2D^4\}$ and $\hat{\sigma}^2=\sigma^2D^2+L_1^2D^4+2L_1\sigma D^3$.
	\end{theorem}
	
	In the following, we introduce the proof intuition. Lemma~\ref{lemma:stochs_1} shows that the expectation of difference between the true marginal gain $\Delta_t^{(\bm{e})}=f(\bm{x}_t+\bm{e})-f(\bm{x}_t)$ and its surrogate $d_{t}^{(\bm{e})}$ is bounded for all $\bm{e}\in \mathcal{E}$. The proof is inspired by Lemmas~1 and~2 in~\cite{mokhtari2017conditional}.
	
	\begin{lemma}\label{lemma:stochs_1}
		In the $t$-th iteration of generalized LDGM with $\gamma=1$ and $\rho_t = \frac{4}{(t+8^{2/3})}$, for all $\bm{e}\in \mathcal{E}$, we have
		\begin{align}
		\mathbb{E}\left[(d_{t}^{(\bm{e})}-\Delta_t^{(\bm{e})})^2 \right] \leq \frac{Q}{l^2(t+9)^{2/3}}.
		\end{align}
	\end{lemma}
	
	By Lemma~\ref{lemma:stochs_1}, we can bound the error induced by stochastic settings, and thus we can prove a stochastic version of Lemma~\ref{lemma2}, as presented in Lemma~\ref{lemma:stoch_2}. The difference in the proof is mainly about the interplay between the true marginal gain $\Delta_t^{(\bm{e})}$ and its surrogate $d_t^{(\bm{e})}$.
	
	\begin{lemma}\label{lemma:stoch_2} Let $\bm{v}^*$ be the best solution one can achieve using $l$ vectors in $\mathcal{E}$, denoted as $\bm{v}^*=\sum\nolimits_{i=1}^l \bm{e}_i$, where $\bm{e}_i \in \mathcal{E}$. In the $t$-th iteration of generalized LDGM with $\gamma=1$ and $\rho_t = \frac{4}{(t+8^{2/3})}$, the expectation of the increment on $f$ is
		\begin{equation}
		\begin{aligned}
		\mathbb{E}[f(\bm{x}_{t+1})-f(\bm{x}_t)]\geq \frac{\beta}{l}\mathbb{E}[f(\bm{v}^*)-f(\bm{x}_t)] -\frac{2Q^{\frac{1}{2}}}{l(t+9)^{\frac{1}{3}}}.
		\end{aligned}
		\end{equation}
	\end{lemma}

	By Lemma~\ref{lemma:stoch_2}, the approximation guarantee in Theorem~\ref{the_stoch} can be proved in the way similar to the proof of Theorem~\ref{the1_LDGM}.

\section{Empirical Study}

In this section, we empirically examine the performance of generalized LDGM with three sets of experiments, all using real-world data sets. We use two applications: budget allocation with continuous assignments and continuous generalization of maximum coverage~\cite{bian2017guaranteed}, with differentiable DR-submodular and non-differentiable submodular objective functions, respectively. First, we compare generalized LDGM with the gradient-based algorithms, i.e., FW~\cite{bian2017guaranteed} and SCG~\cite{mokhtari2017conditional}, on budget allocation in noise-free, noisy and stochastic settings, as shown in Figure~\ref{fig:ba}. To examine the performance of LDGM-G, we next compare it with generalized LDGM, FW and SCG on the submodular application of maximum coverage, as shown in Figure~\ref{fig:mc}. Finally, we examine the influence of the number of vertices on generalized LDGM, as shown in Figure~\ref{fig:vv}. For notational convenience, generalized LDGM will be simply called LDGM in the figures.

In the experiments, the number of iterations is set to 60 for all algorithms, as using more iterations will not bring much improvement. For generalized LDGM and SCG, the averaging parameter $\rho_t$ is set to $4/(t+8)^{2/3}$ suggested by theoretical analysis. For generalized LDGM under noise-free environments, the look-ahead parameter $\gamma$ is always set to 1. Under additive noise or in stochastic settings, the behavior of algorithms is randomized, and thus the running of each algorithm is repeated for 50 times independently, and the average results are reported.

\subsection{Budget Allocation}\label{sec-exp-budget}

Let a bipartite graph $G=(V,T;E)$ represent a social network, where each source node in $V$ is a marketing channel, each target node in $T$ is a customer, and $E \subseteq V \times T$ is the edge set. The goal of budget allocation is to distribute the budget $k$ among the source nodes such that the expected number of target nodes that get activated is maximized. The allocation of the budget can be represented by a vector $\bm{x} \in  \mathbb{R}^n_+ $, where $x_i$ is the budget allocated to $v_i \in V$. Each source node $v_i \; (i \in [n])$ has a probability $p_i \in [0,1]$, and the probability that a target node $t \in T$ gets activated is $f_t(\bm{x})=1-\prod\nolimits_{i: (v_i,t) \in E} (1-p_{i})^{x_i}.$ By the linearity of expectation, the expected number of active target nodes is $f(\bm{x}) = \sum_{t \in T} f_t(\bm{x})$, which is monotone and DR-submodular~\cite{bian2017guaranteed}.

We use two real-world bipartite graphs: \textit{Yahoo! Search Marketing Advertiser Bidding Data}\footnote{\small \url{https://webscope.sandbox.yahoo.com/catalog.php?datatype=a}} and \textit{Wikipedia Adminship Election Data}.\footnote{\small \url{https://snap.stanford.edu/data/wiki-Elec.html}} The \textit{Yahoo!} data set has 1,000 source nodes, 10,475 target nodes and 52,567 edges, and the \textit{WikiElec} data set has 8,300 source nodes, 8,300 target nodes and 114,040 edges. We select $p_i$ uniformly from $[0,0.4]$ at random for each source node. The size constraint, i.e., $|\bm{x}|\leq k \wedge \bm{x}\geq 0$, is used in the stochastic setting, while the settings without noise and with additive noise consider a convex polytope $\mathcal{P}=k\cdot conv(E)$, where $E$ is a set of 100 randomly chosen vertices and $k$ is a parameter to control how large is $\mathcal{P}$.

\begin{figure}[t!]\centering
	\begin{minipage}[c]{\linewidth}\centering
		\includegraphics[width=0.49\linewidth]{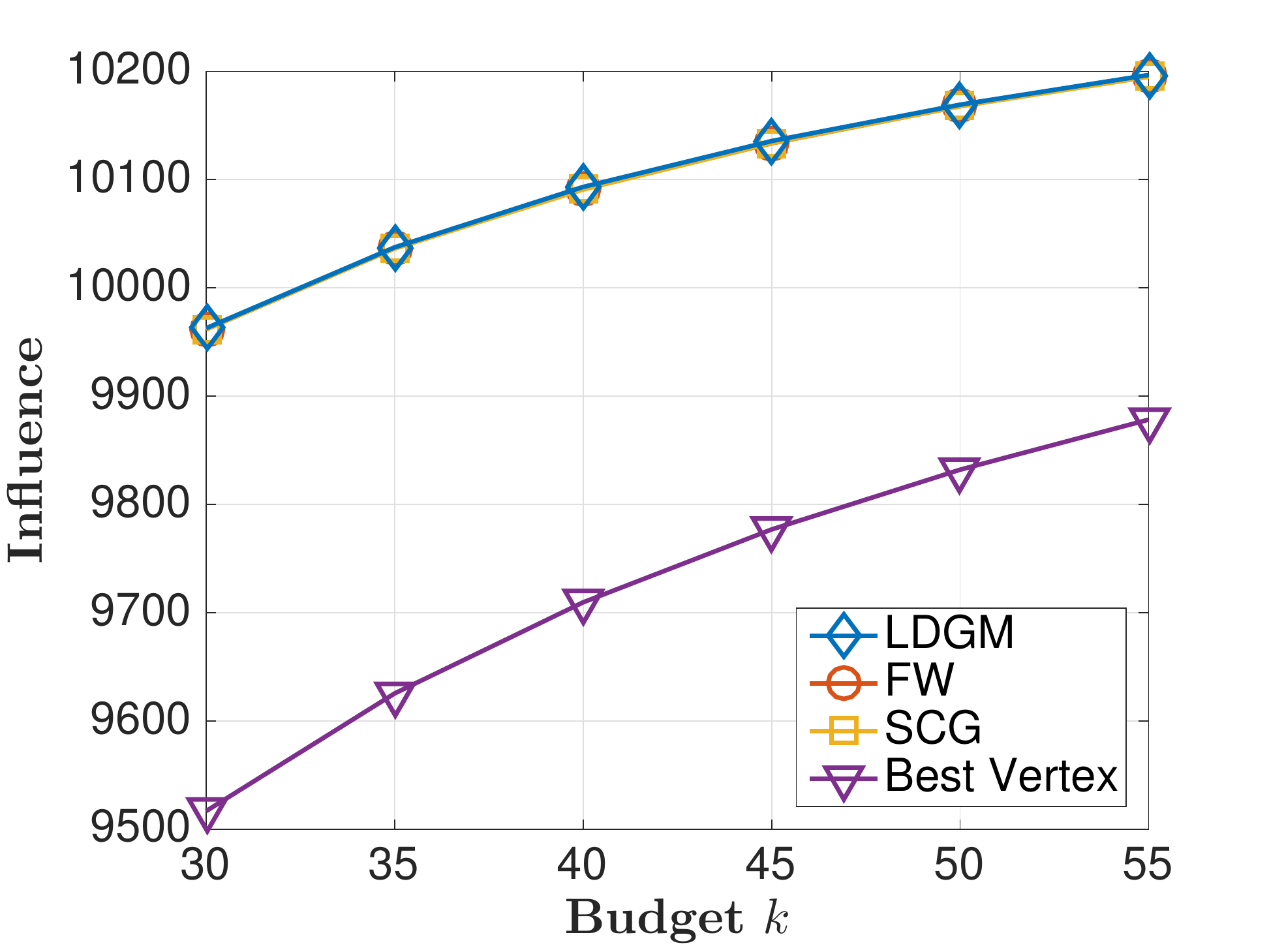}
		\includegraphics[width=0.49\linewidth]{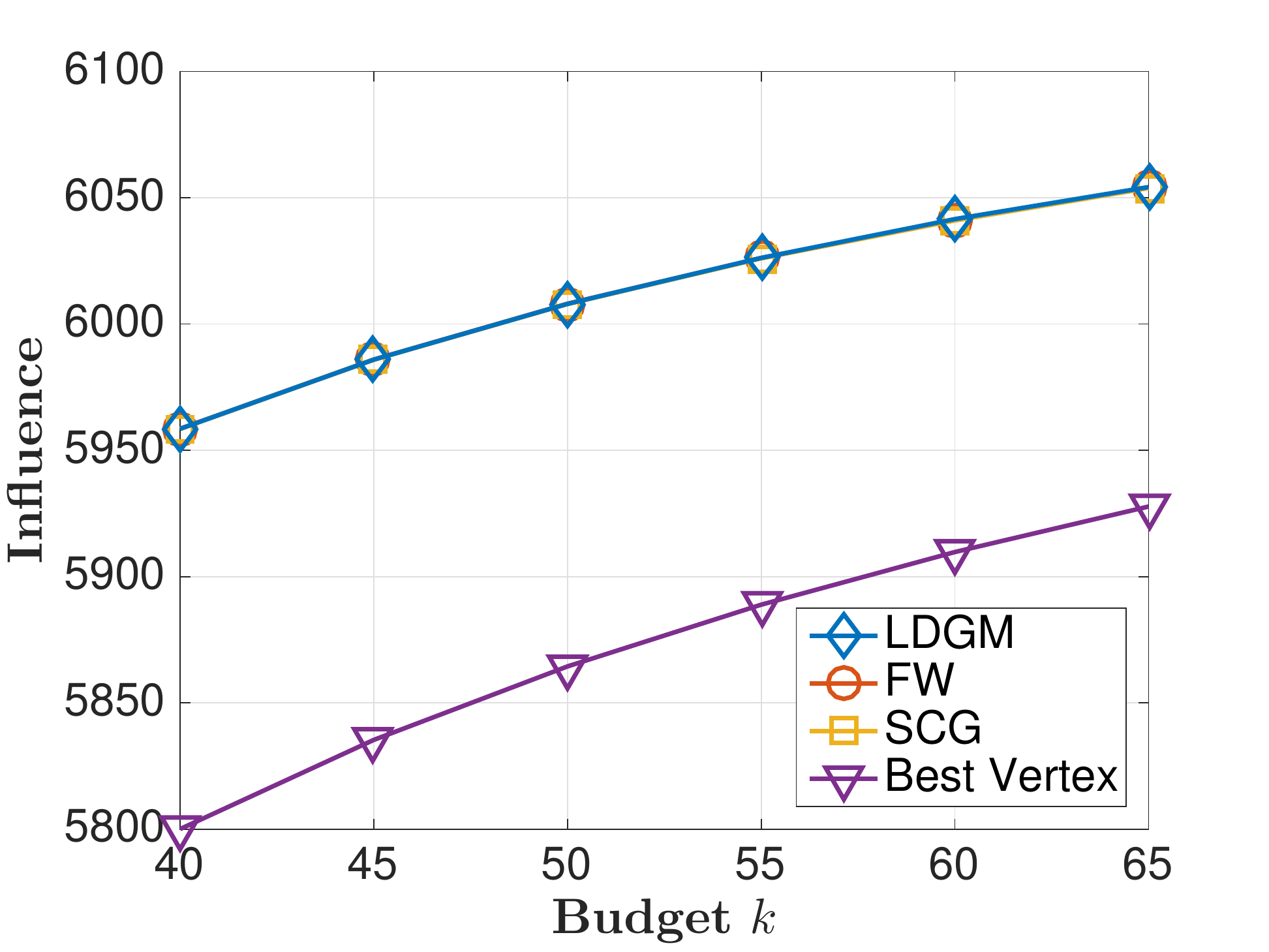}\\
		\small(a) \textit{Noise-free setting}
	\end{minipage}\\ \vspace{0.3em}
	\begin{minipage}[c]{\linewidth}\centering
		\includegraphics[width=0.49\linewidth]{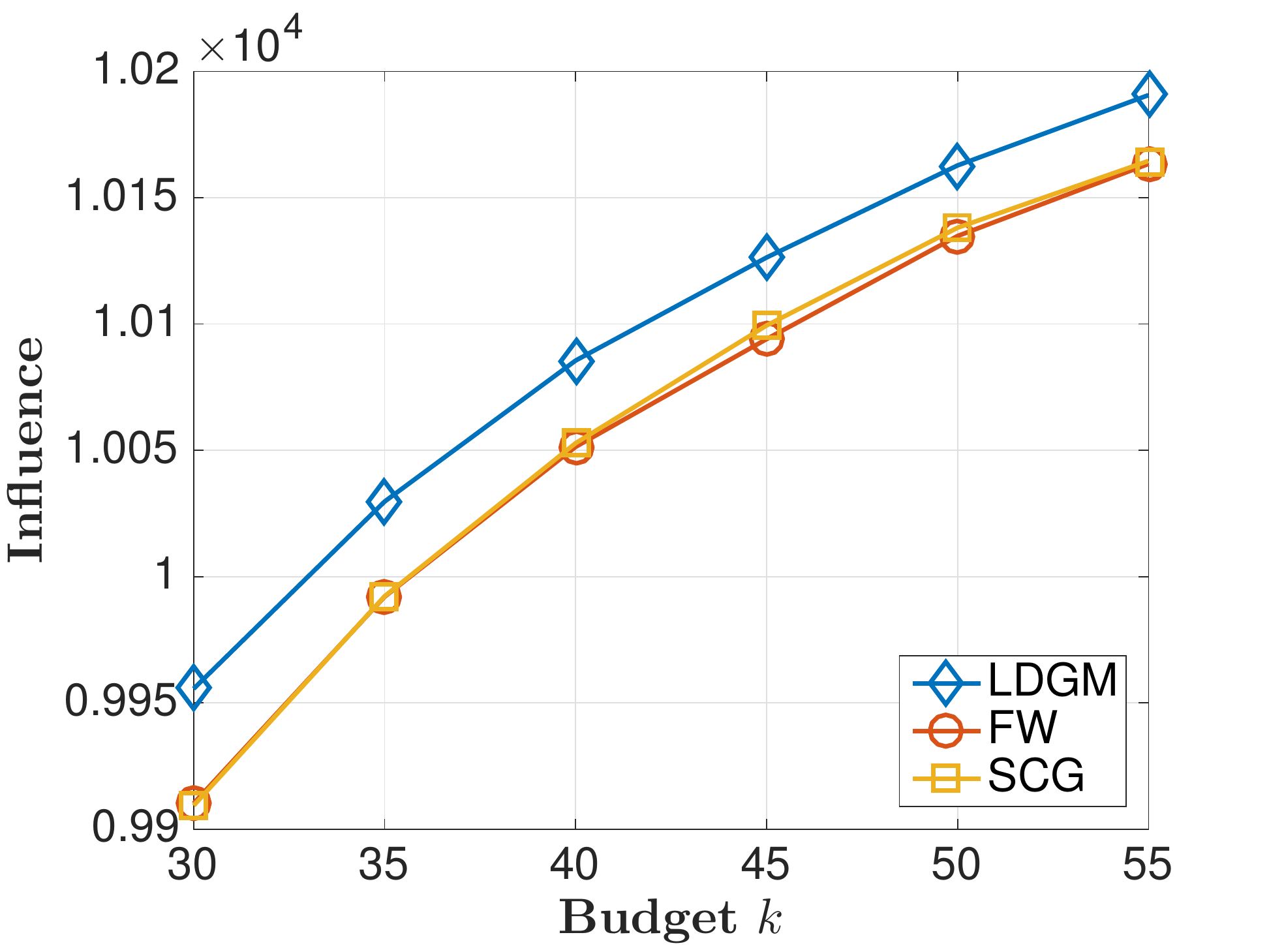}
		\includegraphics[width=0.49\linewidth]{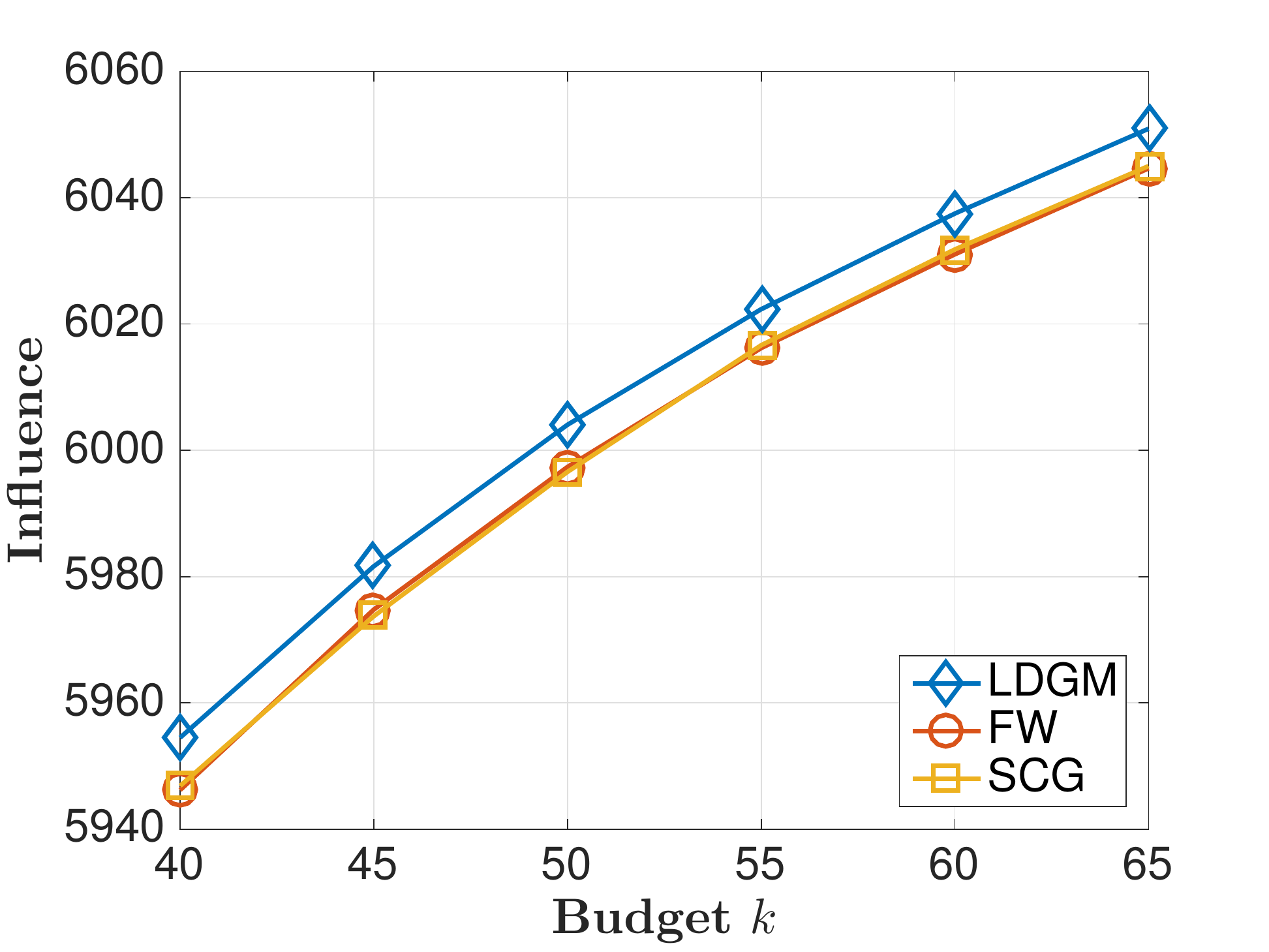}\\
		\small(b) \textit{Additive noise: $\delta=100$ (left) and $\delta=50$ (right)}
	\end{minipage}\\ \vspace{0.3em}
	\begin{minipage}[c]{\linewidth}\centering
		\includegraphics[width=0.49\linewidth]{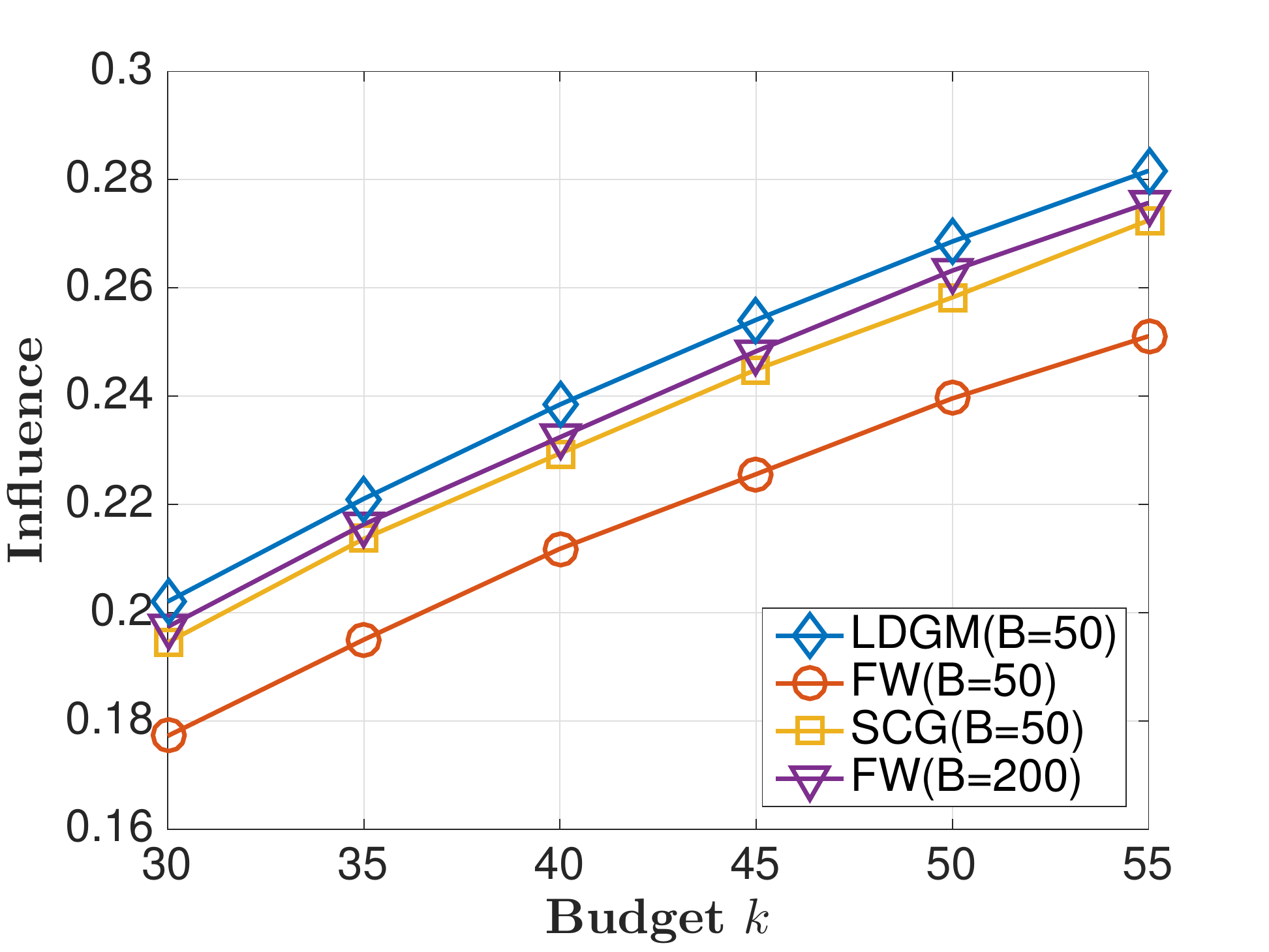}
		\includegraphics[width=0.49\linewidth]{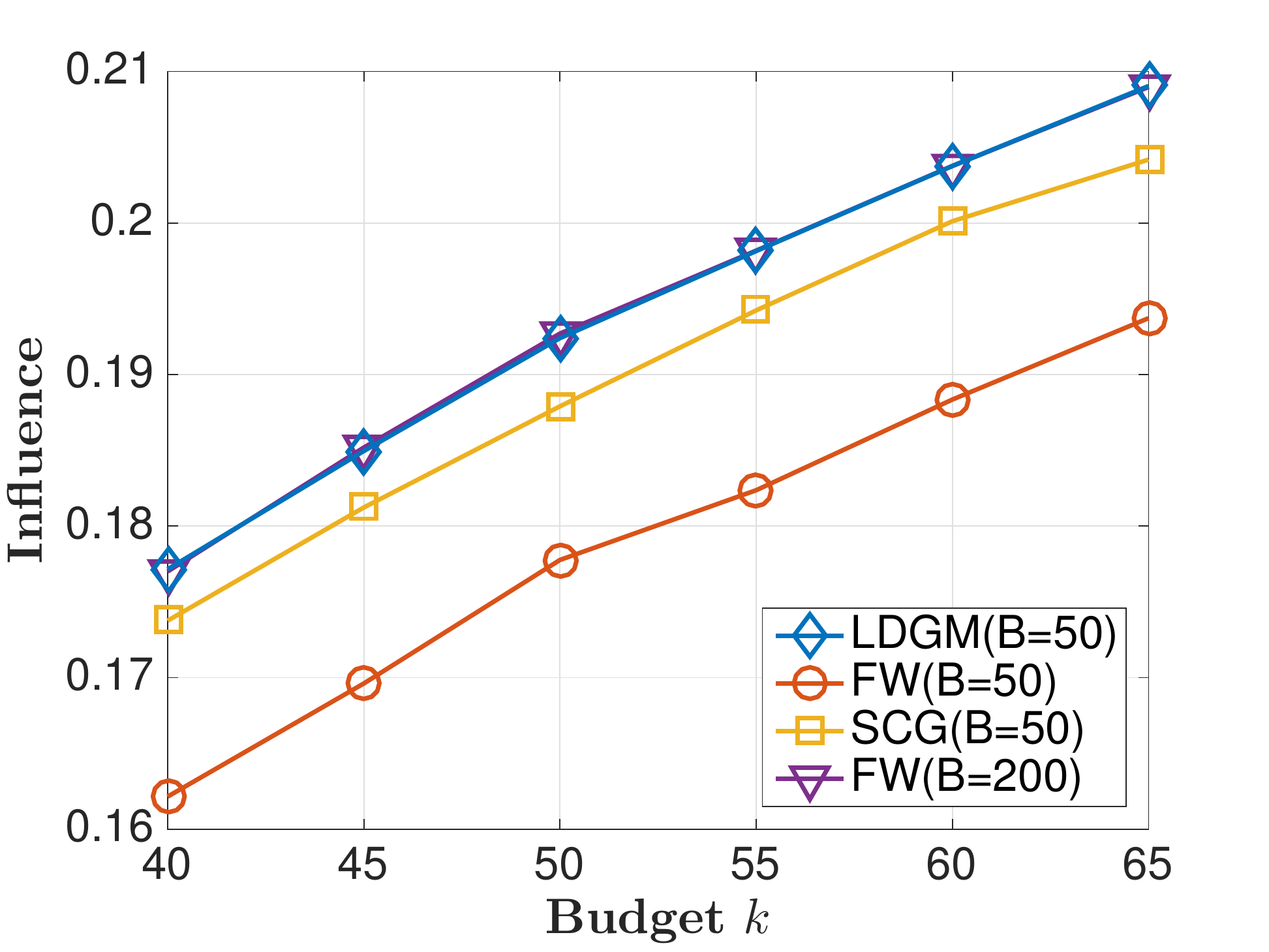}\\
		\small(c) \textit{Stochastic setting}
	\end{minipage}\vspace{-0.3em}
	\caption{Budget Allocation on the \textit{Yahoo!} (left column) and \textit{WikiElec} (right column) data sets.}\label{fig:ba}
\end{figure}

For the noise-free case, generalized LDGM has access to the exact function value, while both FW and SCG have access to the exact gradient. We also enumerate vertices to choose the best one as a baseline. Figure~\ref{fig:ba}(a) shows that generalized LDGM, FW and SCG perform nearly the same, outperforming the best vertex baseline.

For the additive noise case, we consider the black-box situation where only noisy objective function values can be obtained. That is, a function call returns $f(\bm{x})+\epsilon$ instead of $f(\bm{x})$. Assume $\epsilon$ is uniformly randomly chosen from $[-\delta,\delta]$. For generalized LDGM, $\gamma$ is tuned to 46 for the \textit{Yahoo!} data set and 39 for the \textit{WikiElec} data set. For FW and SCG, we use forward difference with step size $a$ to estimate the gradient without incurring too much computation. We tune $a$ to 25 for \textit{Yahoo!} and 18 for \textit{WikiElec}, so that they perform nearly the best. Figure~\ref{fig:ba}(b) shows that generalized LDGM is better than FW, supporting our theoretical analysis in Section~\ref{sec-additive-noise}. The fact that SCG does not outperform FW much under additive noise may be because the gradient estimates calculated by forward difference is biased.

For the stochastic setting, when evaluating the objective $f(\bm{x}) = \mathbb{E}_{t\in T}[f_t(\bm{x})]$, a batch of target nodes $S_B \subseteq T$ with size $B$ are uniformly randomly selected for estimation, i.e., $\tilde{f}(\bm{x}) = \frac{1}{B}\sum_{t\in S_B}f_t(\bm{x})$. Thus, generalized LDGM obtains $\tilde{f}(\bm{x})$, while FW and SCG have access to the gradient of $\tilde{f}(\bm{x})$. The lookahead parameter $\gamma$ of generalized LDGM is to 5. The batch size $B$ is set to 50. Figure~\ref{fig:ba}(c) shows that generalized LDGM performs the best. We can observe that SCG outperforms FW, consistent with previous analyses~\cite{hassani2017gradient,mokhtari2017conditional}. We also test FW with $B=200$, which is much better than that with $B=50$ as expected, but still does not surpass the generalized LDGM algorithm.

\subsection{Maximum Coverage}

Maximum coverage is a classic submodular application, and has been generalized to be continuous submodular recently~\cite{bian2017guaranteed}. Now each subset $C_i \; (1\leq i \leq n)$ has a confidence $x_i \in [0,1]$ and a monotone covering function $p_i: \mathbb{R}_{+} \rightarrow 2^{C_i}$, and the objective function is $f(\bm{x})=|\cup^n_{i=1} p_i(x_i)|$. We generate the covering function $p_i(x_i)$ randomly to go from $\emptyset$ to $C_i$ as $x_i$ increases from 0 to 1. The linear constraint $\bm{a}^T\bm{x}\leq k \wedge \bm{x} \geq \bm{0}$ is used, where $\bm{a}\in \mathbb{R}^n_{+}$ is randomly generated from $(0,50)^n$.

We compare generalized LDGM, FW, SCG and LDGM-G using graph data sets from \textit{BHOSLIB}.\footnote{\small \url{http://sites.nlsde.buaa.edu.cn/~kexu/benchmarks/graph-benchmarks.htm}} Note that the objective $f$ is the size of a set, thus not differentiable, rendering the gradient-based methods not able to be applied directly. Thus, we use forward difference again for FW and SCG, with step size $a=1$. Figure~\ref{fig:mc} plots the results on four data sets, i.e., \textit{frb30-15-1} (450 nodes, 17,827 edges), \textit{frb40-19-1} (760 nodes, 41,314 edges), \textit{frb45-21-1} (945 nodes, 59,186 edges) and \textit{frb50-23-1} (1,150 nodes, 80,072 edges). We can observe that generalized LDGM is better than FW and SCG which suffer from unavailability of unbiased gradients, and LDGM-G performs the best, which is expected as its approximation guarantee in Theorem~\ref{the2_LDGM} is well bounded for submodular cases.

\begin{figure}[t!]\centering
	\begin{minipage}[c]{\linewidth}\centering
		\includegraphics[width=0.49\linewidth]{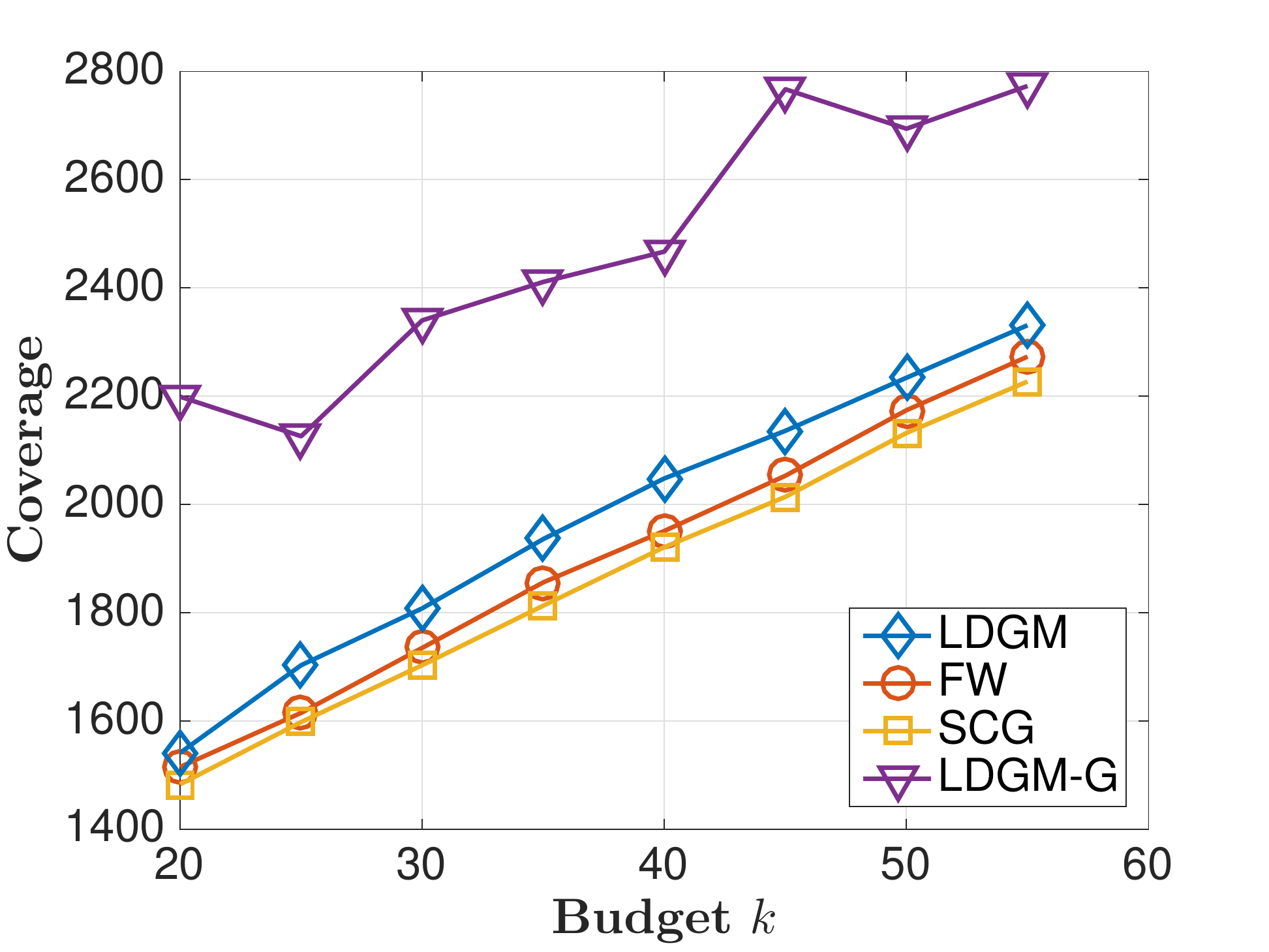}
		\includegraphics[width=0.49\linewidth]{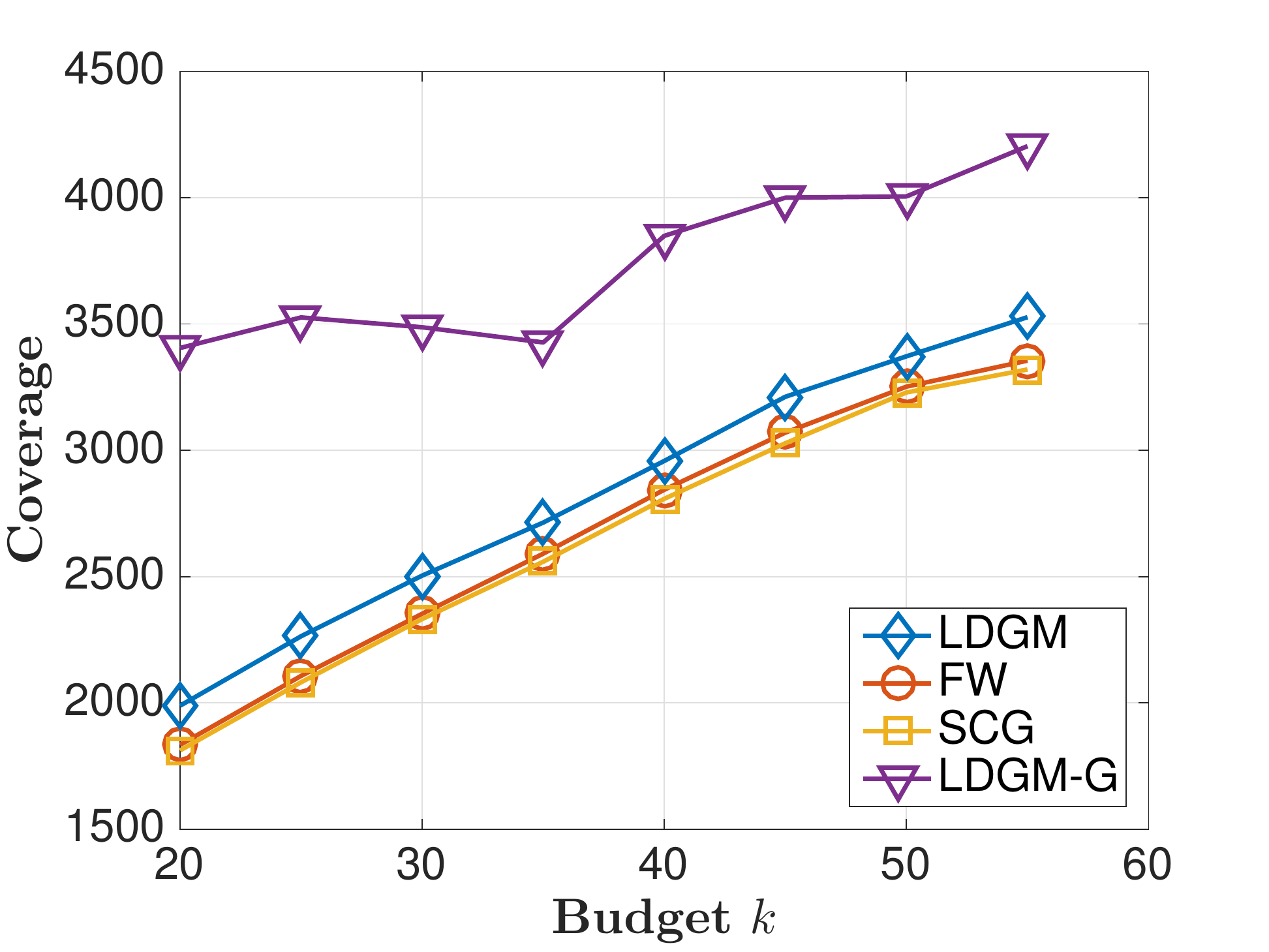}\\
		\small(a) \textit{frb30-15-1} (left) and \textit{frb40-19-1} (right)
	\end{minipage}\\ \vspace{-0.1em}
	\begin{minipage}[c]{\linewidth}\centering
		\includegraphics[width=0.49\linewidth]{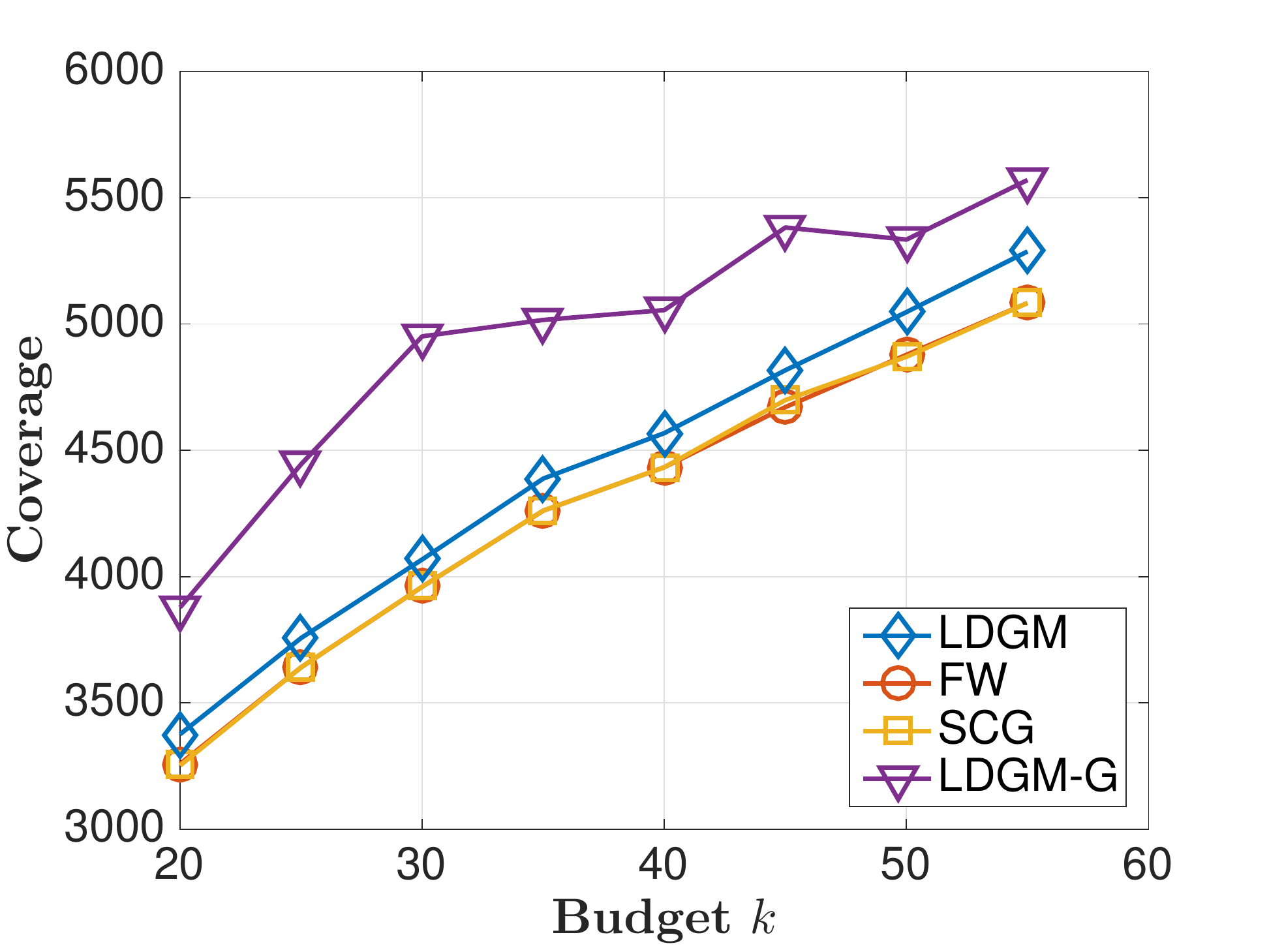}
		\includegraphics[width=0.49\linewidth]{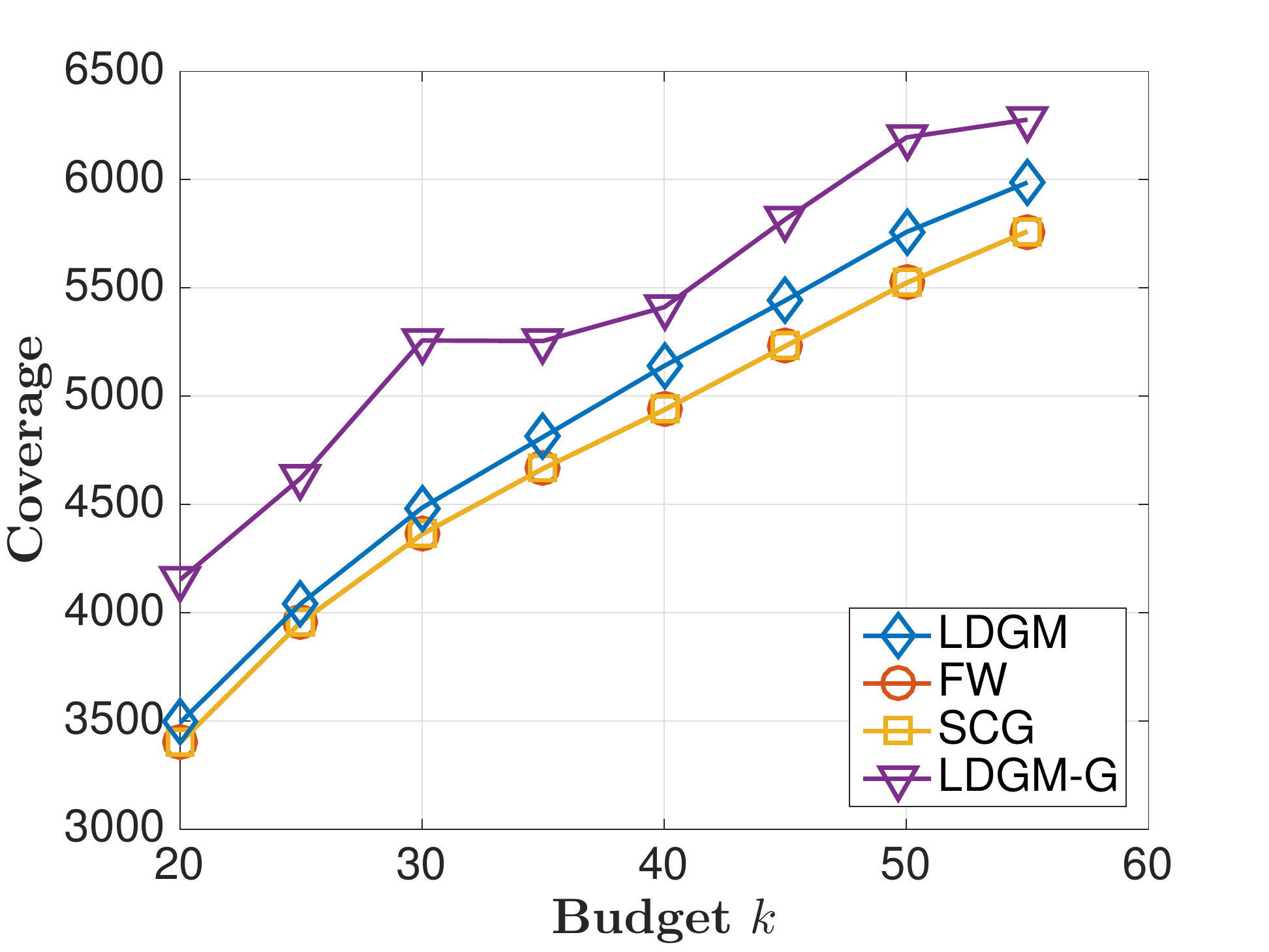}\\
		\small(b) \textit{frb45-21-1} (left) and \textit{frb50-23-1} (right)
	\end{minipage}\\ \vspace{-0.3em}
	\caption{Maximum Coverage on four \textit{BHOSLIB} data sets.}\label{fig:mc}
\end{figure}

\begin{figure}[t!]\centering
	\begin{minipage}[c]{\linewidth}\centering
		\includegraphics[width=0.49\linewidth]{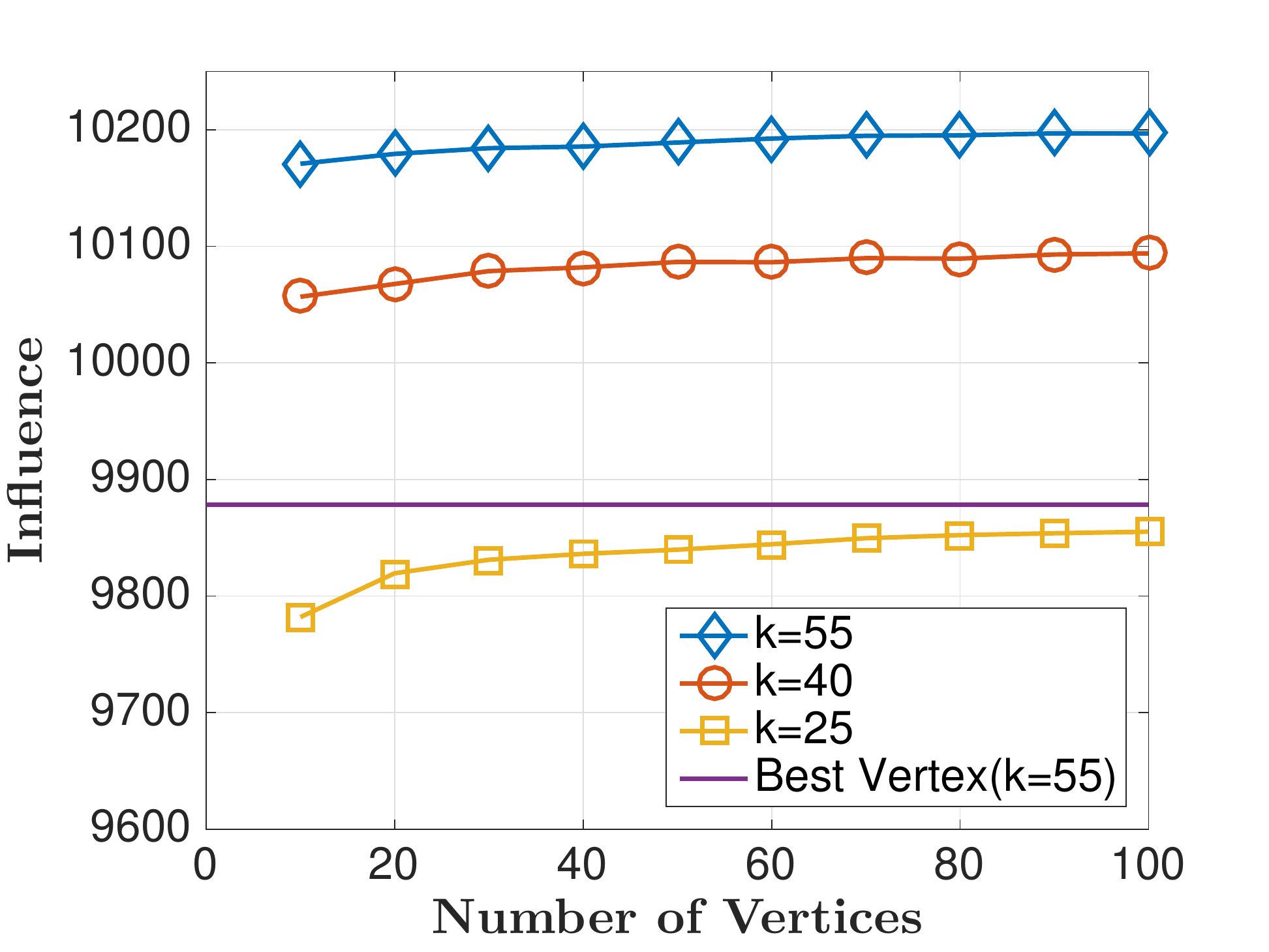}
		\includegraphics[width=0.49\linewidth]{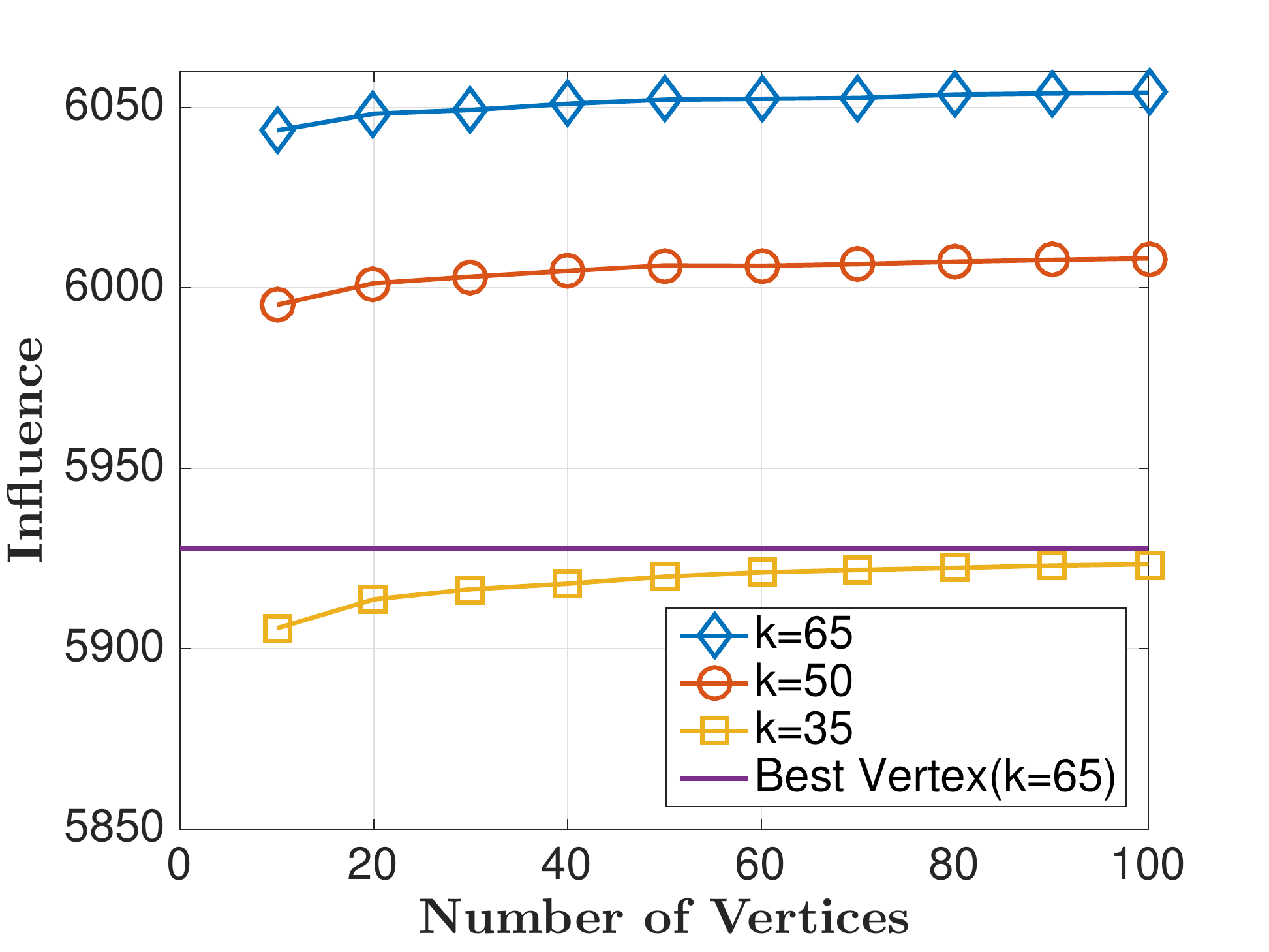}\\
	\end{minipage}\vspace{-0.3em}
	\caption{Generalized LDGM with varying number of vertices on Budget Allocation: \textit{Yahoo!} (left) and \textit{WikiElec} (right).}\label{fig:vv}
\end{figure}

\subsection{Varying Number of Vertices}

As the number of vertices of a convex polytope can be quite large, we examine the performance of generalized LDGM using a subset of vertices instead of all vertices. We use the budget allocation task without noise as in Section~\ref{sec-exp-budget}. The polytope is $\mathcal{P}=k\cdot conv(E)$, where $E$ consists of 100 randomly chosen vertices. We randomly pick a subset of $E$ and run generalized LDGM with different $k$, repeating 40 times to average the results. From Figure~\ref{fig:vv}, we can observe that even with 10\% of all vertices, generalized LDGM performs fairly well. Generalized LDGM runs on discrete lattice in the continuous space. Though few vertices implies sparse lattice, the performance may not degrade much if the objective function satisfies a good Lipschitz property.

\section{Conclusion}

In this paper, we propose a derivative-free algorithm, generalized LDGM, for maximizing monotone (weakly) DR-submodular continuous functions under a convex polytope constraint. With specific parameter settings, generalized LDGM degrades to the LDGM algorithm. In noise-free environments, generalized LDGM can achieve the same approximation guarantee as the best gradient-based algorithm. Under additive noise or in stochastic settings, it can be more robust than gradient-based algorithms. Experiments on budget allocation as well as maximum coverage show the superior performance of generalized LDGM.

\medskip

\small

\bibliographystyle{icml2019}
\bibliography{icml19-ldgm}

\twocolumn[
\icmltitle{Supplementary Material: Maximizing Monotone DR-submodular \\ Continuous Functions by Derivative-free Optimization}
\icmlsetsymbol{equal}{*}




\icmlkeywords{Submodular optimization, non-convex optimization}

\vskip 0.3in
]




\setcounter{lemma}{7}
\setcounter{section}{0}

\section{Detailed Proofs in Section 2}

\begin{lemma}
	The DR-submodularity ratio $\beta$ in Definition~5 is equivalent to $\gamma$ defined in~\cite{hassani2017gradient}, when $f$ is differentiable in $\mathcal{X}$.
\end{lemma}
\begin{proof}
	Let $\bm{x},\bm{y},i$ correspond to the value of $\gamma$. Then,
	\begin{align}
	\gamma&=\frac{[\nabla f(\bm{x})]_i}{[\nabla f(\bm{y})]_i}= \frac{\lim_{k\to 0^+} [f(\bm{x}+k\bm{\chi}_i)-f(\bm{x})]/k}{\lim_{k\to 0^+} [f(\bm{y}+k\bm{\chi}_i)-f(\bm{y})]/k}\\
	&=\lim_{k\to 0^+} \frac{f(\bm{x}+k\bm{\chi}_i)-f(\bm{x})}{f(\bm{y}+k\bm{\chi}_i)-f(\bm{y})}\geq \beta,
	\end{align}
	where the last equality holds as we do not need to care about $\frac{0}{0}$ or $\frac{\infty}{\infty}$, and the inequality holds by the definition of $\beta$.
	
	Next, we prove that $\gamma \leq \beta$. For any $\bm{x}\leq \bm{y},i\in [n], k\in \mathbb{R}_+$, we have
	\begin{align}
	\frac{f(\bm{x}+k\bm{\chi}_i)-f(\bm{x})}{f(\bm{y}+k\bm{\chi}_i)-f(\bm{y})} &=\frac{\int_{0}^{k}[\nabla f(\bm{x}+t\bm{\chi}_i)]_i\mathrm{d}t}{\int_{0}^{k}[\nabla f(\bm{y}+t\bm{\chi}_i)]_i\mathrm{d}t} \\
	&\geq \frac{\gamma\int_{0}^{k}[\nabla f(\bm{y}+t\bm{\chi}_i)]_i\mathrm{d}t}{\int_{0}^{k}[\nabla f(\bm{y}+t\bm{\chi}_i)]_i\mathrm{d}t} \\
	&=\gamma,
	\end{align}
	where the inequality holds by the definition of $\gamma$.
	
	Thus, the lemma holds, i.e., $\beta=\gamma$.
\end{proof}

\section{Detailed Proofs in Section 3}

This section provides detailed proofs for the analysis of LDGM and LDGM-G.


\subsection{Detailed Proofs in Section 3.1}

\begin{myproof}{Lemma~1}
	Let $\bm{v}(i) =  v_i\bm{\chi}_i$. Then, we have
	\begin{align}
	&\frac{f(\bm{x}+\bm{v})-f(\bm{x})}{f(\bm{y}+\bm{v})-f(\bm{y})}\\
	& = \frac{ \sum_{j=1}^{n} f(\bm{x}+\sum_{i=1}^{j}\bm{v}(i))-f(\bm{x}+\sum_{i=1}^{j-1}\bm{v}(i))}{\sum_{j=1}^{n} f(\bm{y}+\sum_{i=1}^{j}\bm{v}(i))-f(\bm{y}+\sum_{i=1}^{j-1}\bm{v}(i))} \\
	& \geq \frac{ \beta \cdot \sum_{j=1}^{n} f(\bm{y}+\sum_{i=1}^{j}\bm{v}(i))-f(\bm{y}+\sum_{i=1}^{j-1}\bm{v}(i))}{\sum_{j=1}^{n} f(\bm{y}+\sum_{i=1}^{j}\bm{v}(i))-f(\bm{y}+\sum_{i=1}^{j-1}\bm{v}(i))} \\
	&=\beta,
	\end{align}
	where the inequality is derived from the definition of DR-submodularity ratio, i.e., Definition~5.
\end{myproof}

\begin{myproof}{Lemma~2}
	Let $\bm{e}^*\in \arg\max_{\bm{e}\in\mathcal{E}}f(\bm{x}+\bm{e})$. Then, we have
	\begin{align}
	& f(\bm{v}^*)-f(\bm{x}) \\
	& \leq f(\bm{x}+\bm{v}^*)-f(\bm{x}) = f\left(\bm{x}+\sum\nolimits_{i=1}^l \bm{e}_i\right)-f(\bm{x}) \\
	& =\sum\nolimits_{k=1}^l f\left(\bm{x}+\sum\nolimits_{i=1}^k \bm{e}_i\right)-f\left(\bm{x}+\sum\nolimits_{i=1}^{k-1} \bm{e}_i\right) \\
	& \leq \frac{1}{\beta} \sum\nolimits_{k=1}^l f(\bm{x}+ \bm{e}_k)-f(\bm{x}) \\
	& \leq \frac{l}{\beta}\cdot (f(\bm{x}+\bm{e}^*)-f(\bm{x})),
	\end{align}
	where the first inequality holds by the monotonicity of $f$, and the second holds by Lemma~1.
\end{myproof}

\begin{myproof}{Lemma~3}
	We prove the lemma by induction on $m\in \mathbb{Z}_+$.
	
	When $m=1$, the lemma obviously holds. When $m=2$, assume that there exists $\bm{x}'=\theta_1\bm{x}_1+\theta_2\bm{x}_2\in Frontier(\mathcal{P})$ where $\theta_1, \theta_2 >0$ and $\theta_1+\theta_2=1$, but $\exists \bm{z}\in conv(X)$ such that $\bm{z}\notin Frontier(\mathcal{P})$; that is, the lemma does not hold. We next show this makes a contradiction. Let $\bm{z}=\eta_1\bm{x}_1+\eta_2\bm{x}_2$, where $\eta_1,\eta_2\geq 0$ and $\eta_1+\eta_2=1$. Since $\bm{z}\notin Frontier(\mathcal{P})$, there exists $\bm{z}'\in \mathcal{P}$ such that $\bm{z}'>\bm{z}$. Let $\bm{\Delta z}=\bm{z}'-\bm{z}$ and $\bm{x}''=\bm{x}'+\xi \bm{\Delta z}=(\theta_1-\xi \eta_1)\bm{x}_1+(\theta_2-\xi \eta_2)\bm{x}_2+\xi \bm{z}'$. Since $\theta_1,\theta_2>0$ and $\eta_1, \eta_2\geq 0$, there exists a sufficiently small $\xi>0$ that makes both $\theta_1-\xi \eta_1$ and $\theta_2-\xi \eta_2$ greater than $0$, and thus $\bm{x}''$ is a convex combination of $\bm{x}_1, \bm{x}_2, \bm{z}'$. As $\bm{x}_1, \bm{x}_2, \bm{z}'\in \mathcal{P}$, we have $\bm{x}''\in \mathcal{P}$. Furthermore, $\bm{x}''>\bm{x}'$. Thus, $\bm{x}'\notin Frontier(\mathcal{P})$, making a contradiction.
	
	Next we consider $m \geq 3$. Let $X=\{ \bm{x}_1,\ldots,\bm{x}_m \}$ and $\bm{x}'=\sum_{i=1}^{m} \theta_i \bm{x}_i \in Frontier(\mathcal{P})$. We want to prove that for any $\bm{x}=\sum_{i=1}^{m} \eta_i \bm{x}_i\in conv(X)$, where $\sum_{i=1}^{m} \eta_i=1$ and $\eta_i\geq 0$, it holds that $\bm{x}\in Frontier(\mathcal{P})$. Let $\bm{y}=\bm{x}'-\xi (\bm{x}-\bm{x}')$, where $\xi>0$. Then, we have $\bm{y}=(\sum_{i=1}^{m} \theta_i \bm{x}_i) - \xi (\sum_{i=1}^{m} (\eta_i-\theta_i) \bm{x}_i)=\sum_{i=1}^{m} ((1+\xi)\theta_i-\xi \eta_i)\bm{x}_i$. As $\theta_i>0$ and $\eta_i\geq 0$, we can find a sufficiently small $\xi>0$, such that $\forall i\in [m]: (1+\xi)\theta_i-\xi \eta_i>0$, i.e., $\bm{y}\in conv(X)$. From the definition of $\bm{y}$, we have $\bm{x}'=\frac{1}{1+\xi}\bm{y}+\frac{\xi}{1+\xi}\bm{x}$. Note that $\bm{x},\bm{y}\in conv(X)\subseteq \mathcal{P}$, $\bm{x}'\in Frontier(\mathcal{P})$ and $\xi>0$. Since the lemma with $m=2$ has been proved, we have $conv(\{\bm{x, \bm{y}}\})\subseteq Frontier(\mathcal{P})$, implying that $\bm{x}\in Frontier(\mathcal{P})$ and thus $conv(X)\subseteq Frontier(\mathcal{P})$.
\end{myproof}

\begin{myproof}{Lemma~4}
	Let $l\cdot\mathcal{E}=\{\bm{r}_i\mid \bm{r}_i=l\cdot \bm{e}_i,\bm{e}_i\in \mathcal{E}\}$. Suppose $\bm{x}^*\notin conv(l\cdot \mathcal{E})$. Since $\bm{x}^* \in \mathcal{P}=conv(E)$, there exist $\bm{y}_1,\ldots,\bm{y}_t \in E$ such that $\bm{x}^*=\sum_{i=1}^t \theta_i \bm{y}_i$, where $\theta_1,\ldots,\theta_t>0$ and $\sum_{i=1}^t \theta_i =1$. As $\bm{x}^*\notin conv(l\cdot \mathcal{E})$, there exists $ i\in [t]: \bm{y}_i\notin l\cdot \mathcal{E}$. By the construction of $\mathcal{E}$ in line~1 of Algorithm~1, we have $\bm{y}_i\notin Frontier(\mathcal{P})$. Thus, we can find $\bm{y}'_i > \bm{y}_i$ in $\mathcal{P}$. Let $\bm{x}'=\sum_{j=1}^{i-1} \theta_j \bm{y}_j+ \theta_i \bm{y}'_i + \sum_{j=i+1}^{t} \theta_j \bm{y}_j>\bm{x}^*$. As $\bm{x}'$ is a convex combination of points in $\mathcal{P}$, $\bm{x}'$ is still in $\mathcal{P}$. In other words, $\bm{x}' \in \mathcal{P}$ and $\bm{x}'>\bm{x}^*$, making a contradiction with $\bm{x}^*\in Frontier(\mathcal{P})$. Thus, $\bm{x}^*\in conv(l\cdot \mathcal{E})$.
	
	As $\bm{x}^*\in conv(l\cdot \mathcal{E})$, there exist $\theta_1,\ldots,\theta_{m'}>0$ with $\sum_{i=1}^{m'} \theta_i =1$ and $m'\leq m$ such that $\bm{x}^*= \sum_{i=1}^{m'} \theta_i \bm{r}_i$. Note that we have made an assumption that those $\theta_i=0$ correspond to the last $m-m'$ terms. We next show that there exists $\bm{x}'=\sum_{i=1}^{m'} \theta_i' \bm{r}_i$ with $\sum_{i=1}^{m'} \theta'_i =1$ and $\theta'_i=k_i\cdot \frac{1}{l}\, (k_i\in \mathbb{Z}_+)$ such that $\forall i \in [m']: |\theta_i-\theta_i'|\leq 1/l$. We first construct an initial $\bm{x}'$ by randomly assigning $k_i\in \mathbb{Z}_+$ subject to $\sum_{i=1}^{m'} \theta'_i =1$. If $\bm{x}'$ does not satisfy $\forall i \in [m']: |\theta_i-\theta_i'|\leq 1/l$, suppose there are in total $T$ positions violating the constraint. Let $j$ denote one position violating the constraint, implying $\theta_j-\theta_j'>1/l$ or $\theta_j'-\theta_j>1/l$. We consider $\theta_j-\theta_j'>1/l$, and the other case can be analyzed similarly. There must exist another position $k$ with $\theta_k-\theta_k'<0$; otherwise, $\sum_{i=1}^{m'} \theta_i>1/l+\sum_{i=1}^{m'} \theta_i'=1+1/l$, making a contradiction. Then, we make such a change: $\theta_j':=\theta_j'+1/l$ and $\theta_k':=\theta_k'-1/l$. Thus, $\theta_j'$ becomes closer to $\theta_j$. For $\theta_k'$, we have $\theta_k-\theta_k'\leq 1/l$, implying that the change at position $k$ will not increase the number $T$ of violations. By repeating this procedure, we can decrease $T$ to 0, i.e., $\bm{x}'=\sum_{i=1}^{m'} \theta_i' \bm{r}_i$ satisfies that $\sum_{i=1}^{m'} \theta'_i =1$ and $\forall i \in [m']: |\theta_i-\theta_i'|\leq 1/l$. Because $\forall i\in [m']: \theta_i>0$ and $\theta'_i$ must be $1/l$ times an integer, we have $\forall i\in [m']: \theta'_i\geq 0$. Let $\bm{v}'=\bm{x}'$. We have
	\begin{align}\label{eq-3}
	&\norm{\bm{x}^*-\bm{v}'} \leq \sum\nolimits_{i=1}^{m'} |\theta_i-\theta_i'|\cdot \norm{\bm{r}_i} \\
	&\leq \frac{1}{l}\cdot \sum\nolimits_{i=1}^{m'} \norm{\bm{r}_i} \leq \frac{m'D}{l} \leq \frac{mD}{l},
	\end{align}
	where the third inequality holds by Assumption~2.

	Finally, we only need to show that $\bm{x}' \in Frontier(\mathcal{P})$. Note that $\forall i\in [m']: \bm{r}_i\in \mathcal{P}$, and $\bm{x}^*=\sum_{i=1}^{m'} \theta_i \bm{r}_i\in Frontier(\mathcal{P})$ is a convex combination of them with all $\theta_i>0$. According to Lemma~3, $conv(\{\bm{r}_1,\ldots,\bm{r}_{m'}\}) \subseteq Frontier(\mathcal{P})$. As $\bm{x}'\in conv(\{\bm{r}_1,\ldots,\bm{r}_{m'}\}) $, we have $\bm{x}'\in Frontier(\mathcal{P})$. Thus, the lemma holds.
\end{myproof}


\begin{myproof}{Theorem~1}
	Let $\bm{x}^*\in Frontier(\mathcal{P})$ denote a global optimal solution, i.e., $f(\bm{x}^*)=\mathrm{OPT}$. Let $\bm{v}'\in Frontier(\mathcal{P})$ be the solution suggested by Lemma~4. Let $\bm{v}^*$ be the best solution one can achieve using the sum of $l$ vectors in $\mathcal{E}$. By Lemma~4 and Assumption~1, we have
	$$ f(\bm{x}^*)-f(\bm{v}^*)\leq f(\bm{x}^*)-f(\bm{v}') \leq L\cdot \norm{\bm{x}^*-\bm{v}'} \leq \frac{mDL}{l}, $$
	implying $f(\bm{v}^*) \geq f(\bm{x}^*)- mDL/l = \mathrm{OPT}- mDL/l$.
	
	According to the algorithm procedure and Lemma~2, we have, for any $t\in \{0,1,\ldots,l-1\}$,
	\begin{align}
	f(\bm{x}_{t+1})- f(\bm{x}_{t}) &\geq \frac{\beta}{l}(f(\bm{v}^*)-f(\bm{x}_t))\\
	&\geq  \frac{\beta}{l}\left( \mathrm{OPT}- \frac{mDL}{l}-f(\bm{x}_t)\right).
	\end{align}
	By a simple transformation, we can equivalently get
	\begin{align}
	&\mathrm{OPT} - \frac{mDL}{l} - f(\bm{x}_{t+1}) \\
	&\leq \left(1-\frac{\beta}{l}\right)\cdot \left(\mathrm{OPT} - \frac{mDL}{l} - f(\bm{x}_t)\right).
	\end{align}
	Thus, $\mathrm{OPT} - mDL/l - f(\bm{x}_{l}) \leq (1-\beta/l)^l \cdot (\mathrm{OPT} - mDL/l - f(\bm{x}_0)) \leq e^{-\beta} (\mathrm{OPT}- mDL/l)$, leading to $f(\bm{x}_{l}) \geq (1-e^{-\beta})\cdot \mathrm{OPT} - (1-e^{-\beta})mDL/l.$
\end{myproof}

To prove Theorem~2, we need the following lemma, which is about the average gain by adding multiple copies of one vector from $\mathcal{E}$ in each iteration of LDGM-G.

\begin{lemma}\label{lemma:genl}
	Let $\bm{v}^*$ be the best solution one can achieve using $l$ vectors from $\mathcal{E}$ which is an orthogonal set. Let $\bm{x}$ be the solution obtained in one iteration of Algorithm~2, and let $(\bm{e}^*,k)$ denote the combination obtained at line~4 in the next iteration. It holds that
	\begin{align}
	\frac{f(\bm{x}+k\bm{e}^*)-f(\bm{x})}{k}\geq \frac{\alpha}{l}(f(\bm{v}^*)-f(\bm{x})).
	\end{align}
\end{lemma}
\begin{proof}
	Suppose $\mathcal{E}=\{ \bm{e}_1,..., \bm{e}_m \}$. Let $e_{it}$ denote the $t$-th entry of $\bm{e}_i$. For any $\bm{e}\in \mathcal{E} \subseteq \mathbb{R}^n_+$, we have $\bm{e}\geq \bm{0}$, i.e., any entry of $\bm{e}$ is no less than $0$. Since $\mathcal{E}$ is orthogonal, i.e., $\forall i, j: \langle \bm{e}_i,\bm{e}_j \rangle =0$, we have $\forall t\in [n]$, $e_{it}$ and $e_{jt}$ cannot both be greater than 0; otherwise $\langle \bm{e}_i,\bm{e}_j \rangle = \sum_{t=1}^{n} e_{it}\cdot e_{jt}>0$, making a contradiction. Thus, the coordinate-wise maximum `$\vee$' is equivalent to `$+$' for $\bm{e}_i$ and $\bm{e}_j$, i.e., $\bm{e}_i\vee \bm{e}_j=\bm{e}_i+\bm{e}_j$. Let $\bm{v}=\bm{x}\vee \bm{v}^*-\bm{x}$. As $\bm{x}$ is generated in line~9 of Algorithm~2, we can represent $\bm{v}$ by $\sum_{i =1}^m l_i\bm{e}_i$, where $l_i \in \mathbb{Z}_+$. Let $I=\{ i\mid l_i>0 \}$, then we have $\bm{v}=\sum_{i\in I} l_i\bm{e}_i$. We use $I(1:j)$ to denote the set of the first $j$ elements in $I$, and $I_j$ to denote the $j$-th element in $I$. Thus, we have
	\begin{align}
	&f(\bm{v}^*)-f(\bm{x})\\
	& \leq f(\bm{x}\vee \bm{v}^*)-f(\bm{x})\\
	& =f(\bm{x}+\bm{ v})-f(\bm{x})\\
	& = f\left(\bm{x}+\sum\nolimits_{i\in I} l_i \bm{e}_i\right)-f(\bm{x}) \\
	& =\sum_{j=1}^{|I|} f\left(\bm{x}+\!\sum_{i\in I(1:j)} \!l_i\bm{e}_i\right)-f\left(\bm{x}+\!\sum_{i\in I(1:j-1)} \!l_i\bm{e}_i\right) \\
	& \leq \frac{1}{\alpha} \sum_{j=1}^{|I|} \frac{f(\bm{x}+ l_{I_j}\bm{e}_{I_j})-f(\bm{x})}{l_{I_j}}\cdot l_{I_j}\\
	& \leq \frac{1}{\alpha} \displaystyle\sum_{j=1}^{|I|} \frac{f(\bm{x}+ k\bm{e}^*)-f(\bm{x})}{k}\cdot l_{I_j} \\
	&\leq \frac{l}{\alpha} \cdot \frac{f(\bm{x}+ k\bm{e}^*)-f(\bm{x})}{k},
	\end{align}
	where the first inequality holds by the monotonicity of $f$, the second inequality holds by the definition of submodularity ratio in Definition~4 and the orthogonality of $\mathcal{E}$, the third inequality is derived by $(\bm{e}^*,k):=\arg\max_{(\bm{e},j)} \{\frac{ f(\bm{x}+j\bm{e})-f(\bm{x})}{j} \mid 	\bm{e}\in\mathcal{E}, j\in \mathbb{Z}_+: j\leq l -\frac{\langle \bm{x}, \bm{e} \rangle}{\langle \bm{e}, \bm{e} \rangle} \}$ and $l_{I_j} \leq l -\frac{\langle \bm{x}, \bm{e}_{I_j} \rangle}{\langle \bm{e}_{I_j}, \bm{e}_{I_j} \rangle} $, and the last inequality holds by $\sum_{j=1}^{|I|} l_{I_j}\leq l$ due to $\bm{v} \leq \bm{v}^*$.
\end{proof}

\begin{myproof}{Theorem~2}
	Let $\bm{x}^*\in Frontier(\mathcal{P})$ denote a global optimal solution. Let $\bm{v}'\in Frontier(\mathcal{P})$ be the solution suggested by Lemma~4. Let $\bm{v}^*$ be the best solution one can achieve using the sum of $l$ vectors in $\mathcal{E}$. As the analysis in the proof of Theorem~1, we have $f(\bm{v}^*) \geq \mathrm{OPT}- mDL/l$.
	
	Let $\bm{x}^{(s)}$ denote the solution obtained by Algorithm~2 after $s$ iterations. Thus, $\bm{x}^{(0)}=\bm{0}$. Let $(\bm{e}^{(s)},k^{(s)})$ be the combination returned by line~4 in the $s$-th iteration of Algorithm~2. Assume that the algorithm terminates after $t+1$ iterations. According to Lemma~\ref{lemma:genl}, we have, for any $s \leq t$, $$\frac{f(\bm{x}^{(s+1)})-f(\bm{x}^{(s)})}{k^{(s+1)}}\geq \frac{\alpha}{l}(f(\bm{v}^*)-f(\bm{x}^{(s)})).$$
	By rearranging the above inequality, we get $f(\bm{x}^{(s+1)}) \geq \frac{\alpha k^{(s+1)}}{l}f(\bm{v}^*)+(1-\frac{\alpha k^{(s+1)}}{l})f(\bm{x}^{(s)})$. Next we prove by induction that for all $ 0\leq s \leq t+1: f(\bm{x}^{(s)})\geq (1-(1-\alpha \frac{\sum_{i=1}^{s}k^{(i)}}{sl})^s)f(\bm{v}^*)$. It trivially holds for $s=0$. Assume the inequality holds for $s$. Then, we have
	\begin{equation}
	\begin{aligned}
	& f(\bm{x}^{(s+1)}) \geq \frac{\alpha k^{(s+1)}}{l} f(\bm{v}^*)+\left(1-\frac{\alpha k^{(s+1)}}{l}\right)f(\bm{x}^{(s)}) \\
	&\geq \left(1-\left(1-\alpha \frac{\sum_{i=1}^{s}k^{(i)}}{sl}\right)^s \left(1-\frac{\alpha k^{(s+1)}}{l}\right)\right)f(\bm{v}^*)\\
	&\geq \left(1-\left(1-\alpha \frac{\sum_{i=1}^{s+1}k^{(i)}}{(s+1)l}\right)^{s+1}\right)f(\bm{v}^*),
	\end{aligned}
	\end{equation}
	where the second inequality holds by induction, and the last is derived by applying the AM-GM inequality.
	
	Thus, for $\bm{x}^{(t+1)}$, we have
	$$f(\bm{x}^{(t+1)}) \geq \left(1-\left(1-\alpha \frac{\sum_{i=1}^{t+1}k^{(i)}}{(t+1)l}\right)^{t+1}\right)f(\bm{v}^*).$$
	As Algorithm~2 terminates when line~5 is satisfied, we have $\sum_{i=1}^{t}k^{(i)} \leq l$ and $\sum_{i=1}^{t+1}k^{(i)}>l$. By applying the latter one, we have
	\begin{align}
	f(\bm{x}^{(t+1)}) &> \left(1-\left(1-\alpha \frac{l}{l(t+1)}\right)^{t+1}\right)f(\bm{v}^*) \\
	&\geq (1-e^{-\alpha})f(\bm{v}^*).
	\end{align}
	
	However, $\bm{x}^{(t+1)}$ is not a feasible solution. We can represent $\bm{x}^{(t+1)}$ as
	$$
	f(\bm{x}^{(t+1)})=f(\bm{x}^{(t)})+\left( f(\bm{x}^{(t+1)})-f(\bm{x}^{(t)}) \right).
	$$
	Next we analyze the marginal gain $f(\bm{x}^{(t+1)})-f(\bm{x}^{(t)})=f(\bm{x}^{(t)}+k^{(t+1)}\bm{e}^{(t+1)})-f(\bm{x}^{(t)})$. As $\mathcal{E}$ is orthogonal here, $\bm{x}^{(t)}$ can be represented by one unique combination of $\bm{e}_i\in \mathcal{E}$. Suppose $\bm{x}^{(t)}$ has $k'$ copies of $\bm{e}^{(t+1)}$, i.e., $\langle \bm{x}^{(t)}-k'\bm{e}^{(t+1)}, \bm{e}^{(t+1)}\rangle =0$. By applying the definition of $\alpha$, we have
	\begin{align}
	&f(\bm{x}^{(t)}+k^{(t+1)}\bm{e}^{(t+1)})-f(\bm{x}^{(t)})\\
	&\leq \frac{1}{\alpha}\left( f((k^{(t+1)}+k')\bm{e}^{(t+1)})-f(k'\bm{e}^{(t+1)}) \right) \\
	&\leq \frac{1}{\alpha}\left( f(l\bm{e}^{(t+1)})-f(k'\bm{e}^{(t+1)}) \right) \\
	&\leq \frac{1}{\alpha}\left( f(\hat{\bm{x}})-f(k'\bm{e}^{(t+1)}) \right) \leq \frac{1}{\alpha}f(\hat{\bm{x}}) ,
	\end{align}
	where the second inequality holds by $k^{(t+1)} \leq l-\frac{\langle \bm{x}_t, \bm{e}^{(t+1)} \rangle}{\langle \bm{e}^{(t+1)}, \bm{e}^{(t+1)} \rangle}=l-k'$, and the third inequality holds by $\hat{\bm{x}}=\arg\max_{\bm{e}\in \mathcal{E}}  f(l\bm{e})$. Thus, we have
	$$f(\bm{x}^{(t+1)})\leq f(\bm{x}^{(t)})+ \frac{1}{\alpha}f(\hat{\bm{x}}) \leq \frac{1}{\alpha}(f(\bm{x}^{(t)})+ f(\hat{\bm{x}})),$$
	where the last inequality holds by $\alpha \in [0,1]$. Because $f(\bm{x}^{(t+1)}) \geq (1-e^{-\alpha})f(\bm{v}^*)$, we have $$f(\bm{x}^{(t)})+ f(\hat{\bm{x}}) \geq \alpha (1-e^{-\alpha})f(\bm{v}^*),$$ implying $\max\{ f(\bm{x}^{(t)}), f(\hat{\bm{x}})\} \geq (\alpha/2) (1-e^{-\alpha})f(\bm{v}^*)$. Note that the final output solution of Algorithm~2 is at least as good as $\arg\max_{\bm{x} \in \{\bm{x}^{(t)},\hat{\bm{x}}\}} f(\bm{x})$. By applying $f(\bm{v}^*)\geq \mathrm{OPT}-\frac{mDL}{l}$, the theorem holds.
\end{myproof}

\subsection{Detailed Proofs in Section 3.2}

\textbf{Modification}. Let $\mathcal{P}=\mathcal{Q} \cap \mathcal{C}$, where $\mathcal{Q}=\{\bm{x} \mid \bm{a}^T\bm{x} \leq b, \bm{x} \geq \bm{0}\}$ and $\mathcal{C}=\{\bm{x} \mid \bm{0} \leq \bm{x} \leq \bm{c}\}$. Note that if $\bm{c}\in \mathcal{Q}$, the case is trivial as $\mathcal{P}$ degrades to $\mathcal{C}$. We modify LDGM to run on $\mathcal{Q}$ instead of $\mathcal{P}$ (so the number of vertices is $|E|=n+1$ and $|\mathcal{E}|=|Frontier(E)|=n$), and change line~4 of LDGM from ``$\bm{e}^*:=\arg\max_{\bm{e}\in\mathcal{E}}f(\bm{x}_t+\bm{e})-f(\bm{x}_t)$" to ``$\bm{e}^*:=\arg\max_{\bm{e}\in\mathcal{E}, \bm{x}_t+\bm{e} \in \mathcal{C}}f(\bm{x}_t+\bm{e})-f(\bm{x}_t)$" (so the output solution must belong to $\mathcal{Q} \cap \mathcal{C} =\mathcal{P}$). Note that in this case, each $\bm{e} \in \mathcal{E}$ has only one entry that is not zero but a positive value.

With a stronger Assumption~1, i.e., $f(\cdot)$ is Lipschitz continuous with constant $L$ in $\mathcal{P}$ instead of only in $Frontier(\mathcal{P})$, we can similarly prove the same approximation guarantee as Theorem~1. Because Lemmas~1 and~3 have nothing to do with this modification, we only need to concern Lemmas~2 and~4. We first derive Lemma~\ref{lemma:increase_adapted} adapted from Lemma~2.

\begin{lemma}\label{lemma:increase_adapted}
	Given the above modification, let $\bm{v}^*\in \mathcal{P}=\mathcal{Q}\cap \mathcal{C}$ be the best solution one can achieve using $l$ vectors in $\mathcal{E}$, denoted as $\bm{v}^*=\sum\nolimits_{i=1}^l \bm{e}_i$, where $\bm{e}_i \in \mathcal{E}$. For $t<l$, let $\bm{e}^*= \arg\max_{\bm{e}\in\mathcal{E}, \bm{x}_t+\bm{e}\in \mathcal{C}}f(\bm{x}_t+\bm{e})$, then we have
	\begin{align}
	f(\bm{x}_t+\bm{e}^*)-f(\bm{x}_t)\geq \frac{\beta}{l}\left(f(\bm{v}^*)-f(\bm{x}_t)\right).
	\end{align}
\end{lemma}
\begin{proof}
	By the monotonicity of $f$, we have
	\begin{align}\label{eq-1}
	& f(\bm{v}^*)-f(\bm{x}_t) \leq f(\bm{x}_t\vee \bm{v}^*)-f(\bm{x}_t).
	\end{align}
	Denote $\bm{v}=\bm{x}_t \vee \bm{v}^*-\bm{x}_t=\sum\nolimits_{i=1}^{l'} \bm{e}_i'$, where $\bm{e}_i'\in \mathcal{E}$. As $\bm{x}_t \vee \bm{v}^*\in \mathcal{C}$, we have $\forall i \in [l']: \bm{x}_t+\bm{e}_i'\in \mathcal{C}$. Then, we have
	\begin{align}\label{eq-2}
	& f(\bm{x}_t \vee \bm{v}^*)-f(\bm{x}_t) = f\left(\bm{x}_t+\sum\nolimits_{i=1}^{l'} \bm{e}_i'\right)-f(\bm{x}_t) \\
	& =\sum\nolimits_{k=1}^{l'} f\left(\bm{x}_t+\sum\nolimits_{i=1}^k \bm{e}_i'\right)-f\left(\bm{x}_t+\sum\nolimits_{i=1}^{k-1} \bm{e}_i'\right)\\
	& \leq \frac{1}{\beta} \sum\nolimits_{k=1}^{l'} f(\bm{x}_t+ \bm{e}'_k)-f(\bm{x}_t) \\
	& \leq \frac{l}{\beta}\cdot (f(\bm{x}_t+\bm{e}^*)-f(\bm{x}_t)),
	\end{align}
	where the first inequality holds by Lemma~1, and the second holds by $l'\leq l$ and the definition of $\bm{e}^*$. Combining Eqs.~(\refeq{eq-1}) and~(\refeq{eq-2}), the lemma holds.
\end{proof}

Similarly to Lemma~3, we derive the following lemma to bound the global optimal solution and its closest feasible solution on lattice.

\begin{lemma}\label{lemma:bound_adapted}
	Given the above modification, let $\bm{x}^*\in Frontier(\mathcal{P})$ denote a global optimal solution and $|\mathcal{E}|=m$, which is actually $n$. There exist $\bm{e}_1',\ldots,\bm{e}'_{l'} \in \mathcal{E}$ such that $\bm{v}'=\sum_{i=1}^{l'}\bm{e}_i'\in \mathcal{P}$ and $\norm{\bm{x}^*-\bm{v}'}\leq mD/l$, where $l' \leq l$.
\end{lemma}
\begin{proof}
	We first use the same analysis as in Lemma~4 to construct $\bm{x}'=\sum_{i=1}^{m'} \theta_i' \bm{r}_i$ satisfying $\sum_{i=1}^{m'} \theta'_i =1$ and $\forall i \in [m']: |\theta_i-\theta_i'|\leq 1/l$.  Note that $\bm{x}^*=\sum_{i=1}^{m'} \theta_i \bm{r}_i$, and as we run the algorithm under the constraint $\mathcal{Q}$, $\{\bm{r}_1,\ldots,\bm{r}_{m'}\}$ are actually vertices of $\mathcal{Q}$. Because $\mathcal{C}$ is down-closed and $\bm{x}^*\in \mathcal{C}$, we have $\forall i: \theta_i\bm{r}_i\in \mathcal{C}$. Thus, we have $\theta'_i>\theta_i$, for every $\theta'_i:\theta_i'\bm{r}_i\notin \mathcal{C}$. As $|\theta'_i-\theta_i|\leq 1/l$, we re-assign $\theta'_i:= \theta'_i-1/l$ for every $\theta'_i$ with $\theta'_i\bm{r}_i\notin \mathcal{C}$, so that $\theta'_i\leq \theta_i$ while the property $|\theta'_i-\theta_i|\leq 1/l$ still holds. Now we have a solution  $\bm{x}'=\sum_{i=1}^{m'} \theta_i' \bm{r}_i\in \mathcal{C}$ satisfying $\sum_{i=1}^{m'} \theta'_i \leq1$ and $\forall i \in [m']: |\theta_i-\theta_i'|\leq 1/l$. As $\bm{0}\in \mathcal{Q}$, we can write $\bm{x}'= (1-\sum_{i=1}^{m'}\theta'_i)\bm{0}+ \sum_{i=1}^{m'} \theta_i' \bm{r}_i$, showing that $\bm{x}'$ is a convex combination of points in the convex body $\mathcal{Q}$, i.e., $\bm{x}'\in \mathcal{Q}$. Thus, $\bm{x}'\in \mathcal{Q} \cap \mathcal{C}=\mathcal{P} $. Let $\bm{v}'=\bm{x}'$. As the analysis of Eq.~(\refeq{eq-3}) in the proof of Lemma~4, the lemma holds.
\end{proof}

Following the proof procedure of Theorem~1, we can prove the same approximation guarantee of this modified LDGM for the common convex polytope constraint $\mathcal{P}=\{\bm{x}\mid \bm{a}^T\bm{x} \leq b, \bm{0}\leq \bm{x} \leq \bm{c}\}$. The only difference is that $\bm{v}'$ in Lemma~\ref{lemma:bound_adapted} is not necessarily in $Frontier(\mathcal{P})$, which thus requires a stronger Assumption~1, i.e., $f$ is Lipschitz continuous with constant $L$ in $\mathcal{P}$ instead of only in $Frontier(\mathcal{P})$, to derive $f(\bm{x}^*)-f(\bm{v}')\leq mDL/l$.

\section{Detailed Proofs in Section 4}

This section provides detailed proofs for the analysis of generalized LDGM under noisy environments.

\subsection{Detailed Proofs in Section 4.1}
\begin{myproof}{Theorem~3}
	Let $\bm{x}^*$ denote a global optimal solution, i.e., $f(\bm{x}^*)=\mathrm{OPT}$. We have
	\begin{align}\label{eq-4}
	& f(\bm{x}^*)-f(\bm{x}_t) \leq \langle \bm{v}^*_t, \nabla f(\bm{x}_t)\rangle /\beta\\
	& = \langle \bm{v}^*_t, \nabla f(\bm{x}_t) + \bm{\epsilon}_t \rangle /\beta -\langle \bm{v}^*_t, \bm{\epsilon}_t\rangle  /\beta\\
	&\leq \langle \bm{v}_t, \nabla f(\bm{x}_t) + \bm{\epsilon}_t \rangle /\beta -\langle \bm{v}^*_t, \bm{\epsilon}_t\rangle /\beta \\
	& =\langle \bm{v}_t, \nabla f(\bm{x}_t)  \rangle /\beta + \langle \bm{v}_t-\bm{v}^*_t, \bm{\epsilon}_t\rangle /\beta,
	\end{align}
	where the first inequality holds by the monotonicity of $f$ and the definition of DR-submodularity ratio, and the second one holds by the update policy, i.e., $\bm{v}_t=\arg\max_{\bm{v}\in \mathcal{P}} \langle \bm{v},\nabla f(\bm{x}_t)+\bm{\epsilon}_t\rangle$. As the FW algorithm assumes the Lipschitz continuous property on gradients, i.e., Assumption~3, we have, for any $\bm{x}, \bm{v}$,
	\begin{align}\label{eq-5}
	f(\bm{x})+\frac{1}{l}  \langle \bm{v}, \nabla f(\bm{x})  \rangle -f(\bm{x}+\bm{v}/l) \leq \frac{L_1}{2l^2}.
	\end{align}
	By combining Eqs.~(\refeq{eq-4}) and~(\refeq{eq-5}), and using $\bm{x}_{t+1}=\bm{x}_t+\frac{1}{l}\bm{v}_t$, we have
	\begin{equation}
	\begin{aligned}
	&f(\bm{x}_{t+1}) - f(\bm{x}^*)\\
	&\geq  \left(1-\frac{\beta}{l}\right)( f(\bm{x}_t)-f(\bm{x}^*) )-\frac{L_1}{2l^2} - \frac{1}{l}\langle \bm{v}_t-\bm{v}^*_t, \bm{\epsilon}_t  \rangle.
	\end{aligned}
	\end{equation}
	
	By induction, we can derive that $ f(\bm{x}_{l}) - f(\bm{x}^*)\geq -(1-\beta/l)^l f(\bm{x}^*)- \frac{L_1}{2l} - \frac{1}{l} \sum_{t=0}^{l-1} (1-\beta/l)^{l-1-t} \langle \bm{v}_t-\bm{v}^*_t, \bm{\epsilon}_t  \rangle $. Thus, the theorem holds.
\end{myproof}

\begin{myproof}{Lemma~5}
	The proof is similar to that of Lemma~2. Let $\bm{e}_{i^*_t}= \arg\max_{\bm{e}_i\in\mathcal{E}}f(\bm{x}_t+\gamma\bm{e}_i)+\epsilon_{t,i}$ be the vector chosen in the $t$-th iteration of Algorithm~3, where $i^*_t$ is the order label of the corresponding vector in $\mathcal{E}$. By the monotonicity of $f$ and Lemma~1, we have
	\begin{align}\label{eq-6}
	&f(\bm{v}^*_\gamma)-f(\bm{x}_t) \leq f(\bm{x}_t+\bm{v}^*_\gamma)-f(\bm{x}_t)\\
	& = f\left(\bm{x}_t+\sum\nolimits_{j=1}^{l/\gamma} \gamma\bm{e}_{i_j}\right)-f(\bm{x}_t) \\
	&=\sum_{k=1}^{l/\gamma} f\left(\bm{x}_t+\sum\limits_{j=1}^k \gamma\bm{e}_{i_j}\right)-f\left(\bm{x}_t+\sum\limits_{j=1}^{k-1} \gamma\bm{e}_{i_j}\right) \\
	& \leq \frac{1}{\beta} \sum_{k=1}^{l/\gamma} (f(\bm{x}_t+ \gamma\bm{e}_{i_k})-f(\bm{x}_t) )\\
	&= \frac{1}{\beta}\sum_{k=1}^{l/\gamma} (f(\bm{x}_t+ \gamma\bm{e}_{i_k})+\epsilon_{t,i_k}-\epsilon_{t,i_k}-f(\bm{x}_t))\\
	&  \leq \frac{l}{\gamma \beta} (f(\bm{x}_t+\gamma\bm{e}_{i^*_t})-f(\bm{x}_t))+\frac{1}{\beta} \sum_{k=1}^{l/\gamma} (\epsilon_{t,i^*_t}-\epsilon_{t,i_k}),
	\end{align}
	where the last inequality holds by the definition of $\bm{e}_{i^*_t}$.
	
	Note that we assume $\gamma$ is a positive integer. Thus, we have
	\begin{align}\label{eq-7}
	&\frac{f(\bm{x}_t+\gamma\bm{e}_{i^*_t})-f(\bm{x}_t)}{\gamma} \\
	& =\frac{\sum_{j=1}^{\gamma}f(\bm{x}_t+j\bm{e}_{i^*_t})-f(\bm{x}_t+(j-1)\bm{e}_{i^*_t})}{\gamma}\\
	&\leq \frac{\sum_{j=1}^{\gamma}f(\bm{x}_t+\bm{e}_{i^*_t})-f(\bm{x}_t)}{\gamma\beta} \\
	&= \frac{f(\bm{x}_t+\bm{e}_{i^*_t})-f(\bm{x}_t)}{\beta},
	\end{align}
	where the inequality holds by the definition of DR-submodularity ratio.
	
	By combining Eqs.~(\refeq{eq-6}) and~(\refeq{eq-7}), the lemma holds.	
\end{myproof}

By Lemmas~4 and~5, the proof of Theorem~4 can be accomplished in the same way as Theorem~1. The only difference is that we need to consider the error term due to noise.

\begin{myproof}{Theorem~4}
	Let $\bm{x}^*\in Frontier(\mathcal{P})$ denote a global optimal solution, i.e., $f(\bm{x}^*)=\mathrm{OPT}$. Let $\bm{v}'_\gamma\in Frontier(\mathcal{P})$ be the solution suggested by Lemma~4 with $\mathcal{E}$ and $l$ replaced by $\gamma \cdot \mathcal{E}$ and $l/\gamma$, respectively. Let $\bm{v}^*_\gamma$ be the best solution one can achieve using the sum of $l/\gamma$ vectors in $\gamma\cdot\mathcal{E}$. According to Lemma~4 and Assumption~1, we have
	\begin{align}
	f(\bm{x}^*)-f(\bm{v}^*_\gamma)&\leq f(\bm{x}^*)-f(\bm{v}'_\gamma)\\
	&\leq L\cdot \norm{\bm{x}^*-\bm{v}'_\gamma} \leq mDL\gamma/l,
	\end{align}
	implying $f(\bm{v}^*_\gamma) \geq \mathrm{OPT}- mDL\gamma/l$.
	
	By the algorithm procedure and Lemma~5, we have
	\begin{equation}
	\begin{aligned}
	& f(\bm{x}_{t+1})- f(\bm{x}_{t}) \\
	& \geq \frac{\beta^2}{l}(f(\bm{v}^*_\gamma)-f(\bm{x}_t))-\frac{\beta}{l} \sum_{j=1}^{l/\gamma} (\epsilon_{t,i^*_t}-\epsilon_{t,i_j}) \\
	& \geq  \frac{\beta^2}{l}\left( \mathrm{OPT}\!-\! \frac{mDL\gamma}{l}\!-\!f(\bm{x}_t)\right) -\frac{\beta}{l} \sum_{j=1}^{l/\gamma} (\epsilon_{t,i^*_t}-\epsilon_{t,i_j}).
	\end{aligned}
	\end{equation}
	By a simple transformation, we can equivalently get
	\begin{align}
	&\mathrm{OPT} - \frac{mDL\gamma}{l} - f(\bm{x}_{t+1}) \\
	&\leq \left(1-\frac{\beta^2}{l}\right)\cdot \left(\mathrm{OPT} - \frac{mDL\gamma}{l} - f(\bm{x}_t)\right)\\
	&\quad +\frac{1}{\beta} \sum_{j=1}^{l/\gamma} (\epsilon_{t,i^*_t}-\epsilon_{t,i_j}).
	\end{align}
	By induction on $t$, the theorem holds.
\end{myproof}

\textbf{Comparison between FW and LDGM}. To compare FW with LDGM under additive noise, we compare the last term of their approximation guarantees in Theorems~3 and~4, which are incurred by noise. In the comparison, we consider DR-submodular objective functions, i.e., $\beta=1$. Let $\mathrm{Err}^{FW}_t:=\frac{1}{l}\langle \bm{v}_t-\bm{v}^*_t, \bm{\epsilon}_t   \rangle $ and $\mathrm{Err}^{LDGM}_t:=\frac{1}{l}\sum_{k=1}^{l/\gamma} (\epsilon_{t,i^*_t}-\epsilon_{t,i_k})$. That is, we are to compare $\mathrm{Err}^{FW}_t$ with $\mathrm{Err}^{LDGM}_t$. As LDGM uses the function value and FW uses the gradient, they are not directly comparable. Thus, we consider the case where only noisy function values can be obtained, and use forward difference to estimate the gradient while keeping the two algorithms run with a similar time complexity.

Let $f(\bm{x})$ and $\tilde{f}(\bm{x})$ denote the exact and noisy function values, respectively. Assume the additive noise model: $f(\bm{x})-\epsilon \leq \tilde{f}(\bm{x})\leq f(\bm{x})+\epsilon$. For LDGM, we have $$|\mathrm{Err}^{LDGM}_t|\leq 2\epsilon/\gamma.$$
For FW, we use forward difference as the estimate of gradient: for each $i\in [n]$,
$$\frac{\partial f(\bm{x})}{\partial x_i} \approx \frac{f(\bm{x}+a\bm{\chi}_i)-f(\bm{x})}{a} \approx \frac{\tilde{f}(\bm{x}+a\bm{\chi}_i)-\tilde{f}(\bm{x})}{a}.$$
We have
$$ \left| \frac{\tilde{f}(\bm{x}+a\bm{\chi}_i)-\tilde{f}(\bm{x})}{a}-\frac{f(\bm{x}+a\bm{\chi}_i)-f(\bm{x})}{a}\right| \leq \frac{2\epsilon}{a}, $$
and
\begin{align}
& \left| \frac{f(\bm{x}+a\bm{\chi}_i)-f(\bm{x})}{a}-\frac{\partial f(\bm{x})}{\partial x_i}\right| \\
& = \left| \frac{1}{a}\int_{0}^{a}\frac{\partial f(\bm{x}+t\bm{\chi}_i)}{\partial x_i}\mathrm{d}t-\frac{\partial f(\bm{x})}{\partial x_i}  \right| \\
& \leq \left| \frac{1}{a}\int_{0}^{a}\frac{\partial f(\bm{x}+a\bm{\chi}_i)}{\partial x_i}\mathrm{d}t-\frac{\partial f(\bm{x})}{\partial x_i}  \right| \\
& = \left| \frac{\partial f(\bm{x}+a\bm{\chi}_i)}{\partial x_i}-\frac{\partial f(\bm{x})}{\partial x_i}  \right| \\
& \leq L_1 ||\bm{x}+a\bm{\chi}_i-\bm{x}|| = L_1 a,
\end{align}
where the first inequality holds by the DR-submodularity of $f$ and $\frac{\partial f(\bm{x}+t\bm{\chi}_i)}{\partial x_i}\leq \frac{\partial f(\bm{x})}{\partial x_i}$, and the last inequality holds by the Lipschitz condition on gradient, i.e., Assumption~3. Thus, $|(\bm{\epsilon}_t)_i|=|\frac{\partial f(\bm{x})}{\partial x_i}- \frac{\tilde{f}(\bm{x}+a\bm{\chi}_i)-\tilde{f}(\bm{x})}{a}| \leq \frac{2\epsilon}{a}+L_1a$. By setting the best parameter $a=\sqrt{2\epsilon/L_1}$, we have $|(\bm{\epsilon}_t)_i|\leq 2\sqrt{2\epsilon L_1}$. Thus,
$$ |\mathrm{Err}^{FW}_t| =
\left| \frac{1}{l} \sum_{i=1}^{n} (\bm{v}_t-\bm{v}_t^*)_i\cdot (\bm{\epsilon}_t)_i\right| \leq \frac{4D\sqrt{2\epsilon nL_1}}{l}. $$
We can see that as $\epsilon$ decreases, $\mathrm{Err}^{FW}_t$ shrinks with $O(\sqrt{\epsilon})$ speed, while $\mathrm{Err}^{LDGM}_t$ shrinks with $O(\epsilon)$ speed. Also, by setting large $\gamma$, we may reduce influences from noise to generalized LDGM. In all iterations, we can assume $\epsilon_{t,i}=0$ for all $i$ with $f(\bm{x}_t+\gamma\bm{e}_i)<f(\bm{x}_t+\gamma\bm{e}_{i_t^*})+\epsilon_{t,i^*_t}$, without making any difference to the output of the algorithm. As $\gamma$ becomes greater, the difference among all options $f(\bm{x}_t+\gamma \bm{e}_i)$ becomes more differentiable, implying that less noise terms have actual influences.

\subsection{Detailed Proofs in Section 4.2}

\begin{myproof}{Lemma~6}
	\begin{align}
	&d_{t}^{(\bm{e})}-\Delta_t^{(\bm{e})} =  (1-\rho_t)d_{t-1}^{(\bm{e})} + \rho_t \tilde{\Delta}_t^{(\bm{e})} -\Delta_t^{(\bm{e})} \\
	& = (1-\rho_t)(d_{t-1}^{(\bm{e})}-\Delta_{t-1}^{(\bm{e})}) \\
	& \quad + (1-\rho_t)\Delta_{t-1}^{(\bm{e})} + \rho_t \tilde{\Delta}_t^{(\bm{e})} - \Delta_t^{(\bm{e})}  \\
	& = (1-\rho_t)(d_{t-1}^{(\bm{e})}-\Delta_{t-1}^{(\bm{e})}) \\
	& \quad + (1-\rho_t)(\Delta_{t-1}^{(\bm{e})}- \Delta_t^{(\bm{e})}) + \rho_t (\tilde{\Delta}_t^{(\bm{e})} - \Delta_t^{(\bm{e})}).
	\end{align}
	Taking the square on both sides, and then taking the expectation conditioned on the history $\mathcal{F}_t$ up to iteration $t$, we have
	\begin{align}
	& \mathbb{E}[(d_{t}^{(\bm{e})}-\Delta_t^{(\bm{e})})^2\mid \mathcal{F}_t] \\
	&= (1-\rho_t)^2 \mathbb{E}[(d_{t-1}^{(\bm{e})}-\Delta_{t-1}^{(\bm{e})})^2\mid \mathcal{F}_t] \\
	& \quad + (1-\rho_t)^2 \mathbb{E}[(\Delta_{t-1}^{(\bm{e})}- \Delta_t^{(\bm{e})})^2\mid \mathcal{F}_t] \\
	& \quad + \rho_t^2 \mathbb{E}[(\tilde{\Delta}_t^{(\bm{e})} - \Delta_t^{(\bm{e})})^2 \mid \mathcal{F}_t] \\
	& \quad + 2(1-\rho_t)^2 \mathbb{E}[(d_{t-1}^{(\bm{e})}-\Delta_{t-1}^{(\bm{e})})(\Delta_{t-1}^{(\bm{e})}- \Delta_t^{(\bm{e})})\mid \mathcal{F}_t].
	\end{align}
	Applying the AM-GM inequality to the last term above, i.e., $2(d_{t-1}^{(\bm{e})}-\Delta_{t-1}^{(\bm{e})})(\Delta_{t-1}^{(\bm{e})}- \Delta_t^{(\bm{e})}) \leq \frac{\rho_t}{2}(d_{t-1}^{(\bm{e})}-\Delta_{t-1}^{(\bm{e})})^2 + \frac{2}{\rho_t}(\Delta_{t-1}^{(\bm{e})}- \Delta_t^{(\bm{e})})^2$, we have
	\begin{align}
	& \mathbb{E}[(d_{t}^{(\bm{e})}-\Delta_t^{(\bm{e})})^2\mid \mathcal{F}_t] \\
	& \leq (1-\rho_t)^2 \left(1+\rho_t/2\right) \mathbb{E}[(d_{t-1}^{(\bm{e})}-\Delta_{t-1}^{(\bm{e})})^2\mid \mathcal{F}_t] \\
	& \quad + (1-\rho_t)^2 \left(1+2/\rho_t\right) \mathbb{E}[(\Delta_{t-1}^{(\bm{e})}- \Delta_t^{(\bm{e})})^2\mid \mathcal{F}_t] \\
	& \quad + \rho_t^2 \mathbb{E}[(\tilde{\Delta}_t^{(\bm{e})} - \Delta_t^{(\bm{e})})^2\mid \mathcal{F}_t].
	\end{align}
	Because $(1-\rho_t)^2 \leq 1-\rho_t\leq 1$, we have
	\begin{align}
	& \mathbb{E}[(d_{t}^{(\bm{e})}-\Delta_t^{(\bm{e})})^2\mid \mathcal{F}_t] \\
	& \leq (1-\rho_t/2)\mathbb{E}[(d_{t-1}^{(\bm{e})}-\Delta_{t-1}^{(\bm{e})})^2\mid\mathcal{F}_t] \\
	& \quad + (1+2/\rho_t) \mathbb{E}[(\Delta_{t-1}^{(\bm{e})}- \Delta_t^{(\bm{e})})^2\mid \mathcal{F}_t] \\
	& \quad + \rho_t^2 \mathbb{E}[(\tilde{\Delta}_t^{(\bm{e})} - \Delta_t^{(\bm{e})})^2\mid \mathcal{F}_t].
	\end{align}
	Thus, $\mathbb{E}[(d_{t}^{(\bm{e})}-\Delta_t^{(\bm{e})})^2\mid \mathcal{F}_t]$ is upper bounded by three terms. We leave the first term $(d_{t-1}^{(\bm{e})}-\Delta_{t-1}^{(\bm{e})} )^2$ for induction. For the second term, we have
	\begin{align}
	& \Delta_{t-1}^{(\bm{e})}- \Delta_t^{(\bm{e})}\\
	& = (f(\bm{x}_{t-1}+\bm{e})-f(\bm{x}_{t-1}))-(f(\bm{x}_{t}+\bm{e})-f(\bm{x}_{t}))\\
	& = \langle \nabla f(\bm{x}_{t-1}+\xi_1 \bm{e}), \bm{e} \rangle - \langle \nabla f(\bm{x}_{t}+\xi_2 \bm{e}), \bm{e} \rangle \\
	& = \langle \nabla f(\bm{x}_{t-1}+\xi_1 \bm{e})-\nabla f(\bm{x}_{t}+\xi_2 \bm{e}), \bm{e} \rangle \\
	& \leq || \nabla f(\bm{x}_{t-1}+\xi_1 \bm{e})-\nabla f(\bm{x}_{t}+\xi_2 \bm{e})||\cdot || \bm{e}|| \\
	& \leq L_1 || \bm{x}_{t-1}-\bm{x}_{t}+(\xi_1-\xi_2) \bm{e}||\cdot || \bm{e}||,
	\end{align}
	where the second equality is derived by applying the mean-value theorem with $\xi_1, \xi_2 \in (0,1)$, and the last inequality holds by the Lipschitz condition on gradients, i.e., Assumption~3. Note that $ ||\bm{x}_{t}-\bm{x}_{t-1} || $ and $|| \bm{e}|| $ are upper bounded by $D/l$ according to Assumption~2. Thus, we continue the inequality by
	\begin{align}
	\leq L_1 (D/l+|\xi_1-\xi_2|D/l) \cdot D/l  	\leq 2L_1D^2/l^2,
	\end{align}
	implying $\mathbb{E}[(\Delta_{t-1}^{(\bm{e})}- \Delta_t^{(\bm{e})})^2\mid \mathcal{F}_t] \leq 4L_1^2D^4/l^4$. Next we consider the third term, i.e., $\tilde{\Delta}_t^{(\bm{e})} - \Delta_t^{(\bm{e})}$, introduced by the stochastic process. Applying the mean-value theorem with $\xi_1, \xi_2 \in (0,1)$, we have
	\begin{equation}
	\begin{aligned}
	& \tilde{\Delta}_t^{(\bm{e})} - \Delta_t^{(\bm{e})} = \langle \nabla \tilde{f}(\bm{x}_t+\xi_1\bm{e}) - \nabla f(\bm{x}_t+\xi_2\bm{e}),\bm{e} \rangle \\
	& \leq || \nabla \tilde{f}(\bm{x}_t+\xi_1\bm{e}) - \nabla f(\bm{x}_t+\xi_2\bm{e}) ||\cdot ||\bm{e}|| \\
	&  = || \nabla \tilde{f}(\bm{x}_t+\xi_1\bm{e}) -\nabla f(\bm{x}_t+\xi_1\bm{e}) \\
	& \quad + \nabla f(\bm{x}_t+\xi_1\bm{e}) - \nabla f(\bm{x}_t+\xi_2\bm{e}) ||\cdot ||\bm{e}|| \\
	& \leq || \nabla \tilde{f}(\bm{x}_t+\xi_1\bm{e}) -\nabla f(\bm{x}_t+\xi_1\bm{e}) || \cdot D/l \\
	& \quad + || \nabla f(\bm{x}_t+\xi_1\bm{e}) - \nabla f(\bm{x}_t+\xi_2\bm{e}) ||\cdot D/l \\
	& \leq || \nabla \tilde{f}(\bm{x}_t+\xi_1\bm{e}) -\nabla f(\bm{x}_t+\xi_1\bm{e}) || \cdot D/l + L_1 D^2/l^2,
	\end{aligned}
	\end{equation}
	Taking the square, we have
	\begin{equation}
	\begin{aligned}
	&(\tilde{\Delta}_t^{(\bm{e})} - \Delta_t^{(\bm{e})})^2\\
	&  \leq || \nabla \tilde{f}(\bm{x}_t+\xi_1\bm{e}) -\nabla f(\bm{x}_t+\xi_1\bm{e}) ||^2 \cdot \frac{D^2}{l^2} \\
	& +\frac{L_1^2D^4}{l^4} + \frac{2L_1D^3}{l^3}|| \nabla \tilde{f}(\bm{x}_t+\xi_1\bm{e}) -\nabla f(\bm{x}_t+\xi_1\bm{e}) ||.
	\end{aligned}
	\end{equation}
	Applying Assumption~4 and Jensen's inequality, i.e., $\mathbb{E}[|| \nabla \tilde{f}(\bm{x}) -\nabla f(\bm{x})||]^2 \leq \mathbb{E}[|| \nabla \tilde{f}(\bm{x}) -\nabla f(\bm{x})||^2]$, we have
	\begin{align}
	\mathbb{E}[(\tilde{\Delta}_t^{(\bm{e})} - \Delta_t^{(\bm{e})})^2\mid \mathcal{F}_t] &\leq \frac{\sigma^2D^2}{l^2} +\frac{L_1^2D^4}{l^4} + \frac{2L_1\sigma D^3}{l^3}\\
	& \leq \frac{\sigma^2D^2+L_1^2D^4+2L_1\sigma D^3 }{l^2}.
	\end{align}
	For notational convenience, we denote the numerator as $\hat{\sigma}^2$. Now putting all together, we have
	\begin{align}
	&\mathbb{E}[(d_{t}^{(\bm{e})}-\Delta_t^{(\bm{e})})^2\mid \mathcal{F}_t] \\
	&\leq (1-\rho_t/2)\mathbb{E}[(d_{t-1}^{(\bm{e})}-\Delta_{t-1}^{(\bm{e})})^2\mid \mathcal{F}_t] \\
	& \quad + (1+2/\rho_t) \frac{4L_1^2D^4}{l^4} + \rho_t^2 \frac{\hat{\sigma}^2}{l^2}.
	\end{align}
	Taking the expectation over $\mathcal{F}_t$, and multiplying both sides by $l^2$, we have
	\begin{align}
	& a_t \leq \left(1-\frac{\rho_t}{2}\right)a_{t-1} + \left(1+\frac{2}{\rho_t}\right) \frac{4L_1^2D^4}{l^2} + \rho_t^2\hat{\sigma}^2,
	\end{align}
	where $a_t=l^2\mathbb{E}[(d_{t}^{(\bm{e})}-\Delta_t^{(\bm{e})})^2]$. As $\rho_t=\frac{4}{(t+8)^{2/3}}$, applying Lemma 2 in~\cite{mokhtari2017conditional} leads to
	\begin{align}
	& a_t \leq \frac{Q}{(t+9)^{2/3}},
	\end{align}
	where $Q=\max\{ l^2 (d_{0}^{(\bm{e})}-\Delta_0^{(\bm{e})})^2 9^{2/3}, 16\hat{\sigma}^2+12L_1^2D^4\}$. Note that $(d_{0}^{(\bm{e})}-\Delta_0^{(\bm{e})})^2\leq (\Delta_0^{(\bm{e})})^2 \leq L^2D^2/l^2$ by Assumptions~1 and~2. Thus, we have
	$$Q\leq \max\{L^2D^29^{2/3}, 16\hat{\sigma}^2+12L_1^2D^4\},$$
	implying that the lemma holds.
\end{myproof}

\begin{myproof}{Lemma~7}
	In the $t$-th iteration of Algorithm~3, let $\bm{e}^*= \arg\max_{\bm{e}\in\mathcal{E}} d_{t}^{(\bm{e})}$ be the vector chosen in line~7. Then, we have
	\begin{align}
	& f(\bm{v}^*)-f(\bm{x}_t) \leq f(\bm{x}_t+\bm{v}^*)-f(\bm{x}_t) \\
	& = f\left(\bm{x}_t+\sum\nolimits_{i=1}^l \bm{e}_i\right)-f(\bm{x}_t) \\
	& =\sum\nolimits_{k=1}^l f\left(\bm{x}_t+\sum\nolimits_{i=1}^k \bm{e}_i\right)-f\left(\bm{x}_t+\sum\nolimits_{i=1}^{k-1} \bm{e}_i\right) \\
	& \leq \frac{1}{\beta} \sum\nolimits_{k=1}^l f(\bm{x}_t+ \bm{e}_k)-f(\bm{x}_t) = \frac{1}{\beta} \sum\nolimits_{k=1}^l \Delta_t^{(\bm{e}_k)} \\
	& = \frac{1}{\beta} \sum\nolimits_{k=1}^l d_t^{(\bm{e}_k)} + \frac{1}{\beta} \sum\nolimits_{k=1}^l ( \Delta_t^{(\bm{e}_k)} -d_t^{(\bm{e}_k)}) \\
	& \leq \frac{l}{\beta} d_t^{(\bm{e}^*)} + \frac{1}{\beta} \sum\nolimits_{k=1}^l ( \Delta_t^{(\bm{e}_k)} -d_t^{(\bm{e}_k)}) \\
	& = \frac{l}{\beta} (f(\bm{x}_{t+1})-f(\bm{x}_t))  + \frac{l}{\beta} (d_t^{(\bm{e}^*)}- \Delta_t^{(\bm{e}^*)}) \\
	&  \quad + \frac{1}{\beta} \sum\nolimits_{k=1}^l ( \Delta_t^{(\bm{e}_k)} -d_t^{(\bm{e}_k)}),
	\end{align}
	where the first inequality holds by the monotonicity of $f$, the second holds by Lemma~1, and the third is due to the selection of $\bm{e}^*$. Taking the expectation on both sides,
	\begin{align}
	&\mathbb{E}[f(\bm{v}^*)-f(\bm{x}_t)] \leq \frac{l}{\beta} \mathbb{E}[f(\bm{x}_{t+1})-f(\bm{x}_t)] \\
	& \quad+ \frac{l}{\beta} \mathbb{E} [d_t^{(\bm{e}^*)}- \Delta_t^{(\bm{e}^*)}] + \frac{1}{\beta} \sum\nolimits_{k=1}^l \mathbb{E} [ \Delta_t^{(\bm{e}_k)} -d_t^{(\bm{e}_k)}] \\
	& \leq \frac{l}{\beta} \mathbb{E}[f(\bm{x}_{t+1})-f(\bm{x}_t)] \\
	& \quad + \frac{l}{\beta} \frac{Q^{1/2}}{l(t+9)^{1/3}} + \frac{1}{\beta} \sum\nolimits_{k=1}^l \frac{Q^{1/2}}{l(t+9)^{1/3}} \\
	& = \frac{l}{\beta} \mathbb{E}[f(\bm{x}_{t+1})-f(\bm{x}_t)] + \frac{2Q^{1/2}}{\beta (t+9)^{1/3}},
	\end{align}
	where the second inequality holds by applying Jensen's inequality and Lemma~6. Thus, the lemma holds.
\end{myproof}

\begin{myproof}{Theorem~5}
	Let $\bm{x}^*\in Frontier(\mathcal{P})$ denote a global optimal solution, i.e., $f(\bm{x}^*)=\mathrm{OPT}$. Let $\bm{v}'\in Frontier(\mathcal{P})$ be the solution suggested by Lemma~4. Let $\bm{v}^*$ be the best solution one can achieve using the sum of $l$ vectors in $\mathcal{E}$. As the analysis in the proof of Theorem~1, we have $f(\bm{v}^*) \geq \mathrm{OPT}- mDL/l$.
	
	By Lemma~7, we have
	\begin{align}
	&\mathbb{E}[f(\bm{x}_{t+1})- f(\bm{x}_{t})]\\
	& \geq \frac{\beta}{l} \mathbb{E}[(f(\bm{v}^*)-f(\bm{x}_t))] - \frac{2Q^{1/2}}{l (t+9)^{1/3}}\\
	&\geq  \frac{\beta}{l}\mathbb{E}[ \mathrm{OPT}- mDL/l-f(\bm{x}_t)] - \frac{2Q^{1/2}}{l (t+9)^{1/3}}.
	\end{align}
	By a simple transformation, we can equivalently get
	\begin{align}
	& \mathbb{E}[\mathrm{OPT} -  mDL/l - f(\bm{x}_{t+1})] \\
	& \leq (1-\beta/l) \mathbb{E}[\mathrm{OPT} -  mDL/l - f(\bm{x}_t)]+\frac{2Q^{\frac{1}{2}}}{l (t+9)^{\frac{1}{3}}}.
	\end{align}
	By induction, we have	
	\begin{equation}
	\begin{aligned}
	& \mathbb{E}[\mathrm{OPT} -  mDL/l - f(\bm{x}_{l})]\\
	& \leq (1-\beta/l)^l\mathbb{E}[\mathrm{OPT} - mDL/l- f(\bm{x}_0)]+ \sum_{t=0}^{l-1} \frac{2Q^{\frac{1}{2}}}{l (t+9)^{\frac{1}{3}}} \\
	& \leq e^{-\beta}\mathbb{E}[\mathrm{OPT} - mDL/l]+ \frac{2Q^{\frac{1}{2}}}{l^{\frac{1}{3}}},
	\end{aligned}
	\end{equation}
	leading to
	$$\mathbb{E}[f(\bm{x}_{l})] \geq (1-e^{-\beta})\left(\mathrm{OPT}-\frac{mDL}{l}\right) - \frac{2Q^{\frac{1}{2}}}{l^{\frac{1}{3}}}.$$
\end{myproof}

\end{document}